\documentclass[letterpaper, 10pt, journal, twoside]{IEEEtran}
\makeatletter
\let\NAT@parse\undefined
\makeatother

\usepackage[numbers]{natbib}
\usepackage[bookmarks=true, colorlinks=false,
    linkcolor=black,
    filecolor=mrobota,
    urlcolor=cyan]{hyperref}
\usepackage{amsmath}
\usepackage{amsfonts, amssymb}
\usepackage{algorithm}
\usepackage{caption}
\usepackage{soul}
\usepackage{algorithmic}
\usepackage[dvipsnames]{xcolor}
\usepackage[font=small,skip=3pt]{caption}
\usepackage{subcaption}
\usepackage{adjustbox}
\usepackage{booktabs}
\usepackage{graphicx}
\usepackage{csvsimple}
\usepackage{amsthm}
\usepackage{tikz}
\usepackage{pgf}
\usepackage{import}
\usepackage{pgfplots}
\usepackage{pgfplotstable}
\usepgfplotslibrary{fillbetween}
\usepackage{multirow}
\usepackage{siunitx}
\usepackage{svg}

\newtheorem{lemma}{Lemma}
\newtheorem{theorem}{Theorem}

\usepackage[utf8]{inputenc}

\DeclareMathOperator*{\argmax}{arg\,max}
\DeclareMathOperator*{\argmin}{arg\,min}

\usepackage{stackengine}
\usepackage{color}
\usepackage{todonotes}
\usepackage{amssymb}
\usepackage{mathtools}

\usepackage{orcidlink}
\newcommand{\noboxorcid}[1]{%
  {\hypersetup{pdfborder={0 0 0}}\href{https://orcid.org/#1}{\orcidlink{#1}}}%
}

\usepackage{comment}
\makeatletter
\title{BOW: Bayesian Optimization over Windows \\ for Motion Planning in Complex Environments}

\author{{Sourav Raxit$^{1}$\noboxorcid{0000-0003-1196-2435}}, {Abdullah Al Redwan Newaz$^{1^*}$\noboxorcid{0000-0003-1140-8119}}, {Paulo Padrao$^{3}$\noboxorcid{0000-0003-3966-0279}}, {Jose Fuentes$^{2}$\noboxorcid{0000-0002-6477-5820}}, and {Leonardo Bobadilla$^{2}$\noboxorcid{0000-0003-2097-2432}}
\thanks{Manuscript received: April 08 2025; Revised: June 20 2025 and July 26 2025; Accepted August 12 2025.}
\thanks{This paper was recommended for publication by associate editor Olivier Stasse upon evaluation of the Reviewers' comments.
This work is supported in part by the U.S. EPA grant BR-02F47801-5010M, NSF grants IIS-2024733, IIS-2331908, the ONR grant N00014-23-1-2789, the U.S. Dept. of Homeland Security grant 23STSLA00016-01-00, the U.S. DoD grant 78170-RT-REP, and the FDEP grant INV31.}
\thanks{$^{1}$ S. Raxit, and A. A. R. Newaz are with the Department of Computer Science, University of New Orleans, New Orleans, LA 70148, USA (email: \{sraxit, aredwann\}@uno.edu).
$^{2}$ J. Fuentes, and L. Bobadilla are with the School of Computing and Information Sciences, Florida International University, Miami, FL 33199, USA (email:
        \{jfuen099@, bobadilla@cs.\}fiu.edu).
$^{3}$ P. Padrao is with  Providence College, Department of Mathematics \& Computer Science, Providence,   RI 02918, USA (email:ppadraol@providence.edu).
        }
    \thanks{Digital Object Identifier (DOI): see top of this page.}
}
\begin{document}
    \maketitle
    \begin{abstract}
    This paper introduces the BOW Planner, a scalable motion planning algorithm designed to navigate robots through complex environments using constrained Bayesian optimization (CBO). Unlike traditional methods, which often struggle with kinodynamic constraints such as velocity and acceleration limits, the BOW Planner excels by concentrating on a planning window of reachable velocities and employing CBO to sample control inputs efficiently. This approach enables the planner to manage high-dimensional objective functions and stringent safety constraints with minimal sampling, ensuring rapid and secure trajectory generation. Theoretical analysis confirms the algorithm’s asymptotic convergence to near-optimal solutions, while extensive evaluations in cluttered and constrained settings reveal substantial improvements in computation times, trajectory lengths, and solution times compared to existing techniques. Successfully deployed across various real-world robotic systems, the BOW Planner demonstrates its practical significance through exceptional sample efficiency, safety-aware optimization, and rapid planning capabilities, making it a valuable tool for advancing robotic applications. The BOW Planner is released as an open-source package and videos of real-world and simulated experiments are available at \url{https://bow-web.github.io/}.
\end{abstract}

\begin{IEEEkeywords}
Motion and Path Planning; Collision Avoidance; Nonholonomic Motion Planning
\end{IEEEkeywords}
    \vspace{-10pt}
\section{Introduction}
\IEEEPARstart{M}{otion} planning for mobile robots in complex environments remains a challenging problem in robotics. The task can be formulated as a constrained optimization problem, where the objective is to minimize a cost function (e.g., time, energy) to reach a goal destination, subject to constraints such as obstacle avoidance and kinematic/dynamic feasibility~\cite{Jessica,castaman2021receding}. However, evaluating this constrained objective function over a long planning horizon is computationally expensive, especially for systems with complex dynamics and constraints~\cite{castaman2021receding}.

To address this issue, the Receding Horizon Controller (RHC) approach solves the optimization problem over a shorter, fixed horizon and the solution is applied for a single step before re-optimizing.
Early work in RHC highlighted the critical role of constraints in control processes by meticulously developing a model of the system~\cite{Mayne1991ModelPC}.
Later, a unified motion planning and control solution was introduced, capable of navigating environments with both static and dynamic obstacles~\cite{zhang_improved_2019}.
{
The Dynamic Window Approach (DWA)~\cite{fox1997dynamic} searches robot control commands directly in velocity space, constraining the search to dynamically reachable velocities over short temporal windows. The Hybrid Reciprocal Velocity Obstacle (HRVO)~\cite{snape2011hybrid} method complements this framework by providing closed-form characterization of collision-free velocities. HRVO shifts the velocity obstacle cone to account for mutual agent cooperation, enabling smooth, decentralized avoidance.  
}

{Recent advancements in motion planning include learning-based receding horizon planning~\cite{bergman2020optimization}, risk-aware planning~\cite{nyberg2021risk}, and perception-aware planning~\cite{Costante2016PerceptionawarePP} to enhance performance within computational constraints.
Deep learning-based motion planners fall into two categories: (i) end-to-end collision-free path generation and (ii) subsystem improvement methods~\cite{wang_survey_2021}. End-to-end approaches utilize 3D CNNs, PointNet++, and Transformers for robot locomotion and visuomotor control~\cite{carvalho2024motion, ichnowski2020deep, ren2022robot}. Subsystem frameworks like Dynamic MPNets~\cite{li2021mpc} combine neural planning with model predictive control for robust solutions in learned environments.
Deep learning-based motion planners require prior environment knowledge and struggle to balance optimality, safety, and computational efficiency in uncertain, dynamic environments.}


In motion planning, Bayesian optimization (BO) plays a crucial role in identifying safe and efficient paths by strategically reducing uncertainty within the search space. This is achieved through the combination of BO with Gaussian Processes (GP), which are widely used in high-dimensional motion planning to optimize expensive-to-evaluate black-box functions~\cite{quintero2021motion}. BO employs GP as a surrogate model to iteratively sample the objective function, construct a probabilistic model, and use an acquisition function to identify the next most promising point for evaluation~\cite{ungredda2021bayesian, Diessner2022InvestigatingBO}. Spatial GPs address the curse of dimensionality by selecting a small, spatially diverse subset of the dynamic window, thereby significantly reducing the sample space~\cite{Ali-RSS-23, wilson2012gaussian}, and they avoid overfitting by considering multiple potential curves for the underlying process. {Additionally, GP-based motion planning frameworks, such as GPMP and GPMP2, leverage sparse GP representations to model continuous-time trajectories and perform efficient probabilistic inference via factor graphs, enabling fast optimization under kinodynamic constraints~\cite{mukadam2018continuous}}. While these techniques perform well for unconstrained objectives, most motion planning problems require faster computation and multiple trajectory options under kinodynamic constraints. {To address these challenges, recent methods integrate GP models into model predictive control (MPC) frameworks to capture unmodeled dynamics and propagate uncertainty for safe control~\cite{hewing2019cautious}, while complementary approaches use Bayesian optimization to tune safety margins within MPC under uncertainty, effectively combining real-time constrained optimization with data-efficient learning for robust trajectory planning~\cite{andersson2016model}.}

 In this paper, we propose Bayesian Optimization over Windows (BOW), which incorporates the principles of the Dynamic Window Approach (DWA)~\cite{fox1997dynamic} into the Constrained Bayesian Optimization (CBO) framework~\cite{Gabler2022BayesianOW,ungredda2021bayesian}. Our approach fixes the control input over a short prediction horizon, serving two primary purposes: (i) ensuring computational tractability of the CBO-based optimization, and (ii) aligning with proven motion planning strategies that leverage short-horizon control to balance performance and feasibility.
Unlike earlier methods that rely on heuristics or predefined sampling strategies,  the BOW Planner’s use of the Constrained Expected Improvement (CEI) acquisition function allows it to learn constraints and objectives simultaneously.

The main contributions of this paper are:
(i) \textbf{Enhanced Sample Efficiency:} BOW can find optimal control policies with fewer evaluations, providing feasible solutions in significantly less time compared to existing methods.
(ii) \textbf{Safety-Integrated Optimization:} Safety constraints are directly incorporated into the optimization process, ensuring solutions are both optimal and feasible.
(iii) \textbf{Rapid Planning \& Scalability:} This method outperforms state-of-the-art planning algorithms and scales effectively to tackle high-dimensional motion planning, supporting complex robotic tasks with numerous optimization parameters.
(iv) \textbf{Theoretical Guarantees:} Beyond thorough benchmarking in many different cluttered environments, BOW offers proven guarantees of asymptotic convergence and near-optimal solutions, ensuring reliable performance in demanding settings.

    


    \section{Problem Formulation}
Consider a mobile robot with kinodynamic model operating in a $d$-dimensional environment $\mathcal{W} \subseteq \mathbb{R}^d$ containing $K$ unknown obstacles.
The robot's state space is denoted as $\mathcal{X} \subseteq \mathbb{R}^n$, and its control space is represented by $\mathcal{U} \subseteq \mathbb{R}^m$.


The robot's motion model describes how its state evolves over time based on the applied control inputs.
Given the robot's current state $\mathbf{x}_t \in \mathcal{X}$ at time $t$, and a control input $\mathbf{u}_t \in \mathcal{U}$, the robot's next state $\mathbf{x}_{t+1}$ is computed using the following: 
\begin{equation}
    \mathbf{x}_{t+1} = f(\mathbf{x}_t, \mathbf{u}_t), \label{eqn:motion_model}
\end{equation}
where $f: \mathcal{X} \times \mathcal{U} \to \mathcal{X}$ is the state transition function.


The objective of the motion planning problem is to find a collision-free trajectory $\mathbf{x}: \{0,\ldots, T\} \to \mathcal{X}_{\text{free}}$ that takes the robot from its initial state $\mathbf{x}_{\text{init}} \in \mathcal{X}$ to a desired goal state $\mathbf{x}_{\text{goal}} \in \mathcal{X}_{\text{goal}} \subset \mathcal{X}$, while satisfying the motion model constraints. Here, $\mathcal{X}_{\text{free}} \subseteq \mathcal{X}$ represents the obstacle-free region of the state space, and $T$ is the trajectory duration. Formally, the motion planning problem can be stated as:
\begin{equation}
    \begin{aligned}
        \text{Find} \quad & \mathbf{x}: \{0,\ldots, T\} \to \mathcal{X}_{\text{free}} \\
        \text{s.t.} \quad & \mathbf{x}_0 = \mathbf{x}_{\text{init}}, \,
         \mathbf{x}_T = \mathbf{x}_{\text{goal}}, \,
         \mathbf{x}_{t+1} = f(\mathbf{x}_{t}, \mathbf{u}_{t}).
    \end{aligned}
    \label{eqn:motion_planning_problem}
\end{equation}
To solve the motion planning problem, we can formulate it as an optimization problem by defining a cost-to-go function $J: \mathcal{X} \times \mathcal{U} \to \mathbb{R}$ that measures the quality of the robot's state and control inputs. The cost-to-go function can incorporate various objectives, such as minimizing the distance to the goal state, avoiding obstacles, or optimizing energy consumption. Additionally, we can impose constraints $c_k: \mathcal{X} \to \mathbb{R}$ to ensure the robot's safety and feasibility.
The optimization problem can be expressed as:
\begin{equation}
    \begin{aligned}
        \min_{\mathbf{u}_0, \ldots, \mathbf{u}_{T-1}} \quad &  \sum_{t=0}^{T-1} J(\mathbf{x}_t, \mathbf{u}_t) + J_F(\mathbf{x}_T)\\
        \text{s.t.} \quad &\mathbf{x}_{t + 1} = f(\mathbf{x}_t, \mathbf{u}_t),  \quad  t \in \{0, \ldots, T-1\}, \\
        & c_k(\mathbf{x}_t) \leq 0, \quad  k \in \{1,\ldots, K\}, \; t \in \{1,\ldots, T\},
    \end{aligned}
    \label{eqn:optimization_problem}
\end{equation}
where $K$ is the number of constraints.

In the context of online motion planning, the robot has a limited time budget $\Delta$ to compute its control inputs at each time step. The robot must optimize its trajectory by solving the optimization problem in~\eqref{eqn:optimization_problem} within the allocated time budget. The recommended control input is applied to the robot, and the process repeats at the next timestep.

    \section{BOW Planner}
Given a dynamic planning window $V_d \subset \mathcal{U}$ based on the robot's current velocity and acceleration limits~\cite{fox1997dynamic}, we need to evaluate the objective function $J$ and constraints $c_k$ for each control input $\mathbf{u}  \in V_d$ by simulating the robot's trajectory over a short time interval $[t, t+\Delta)\subseteq [0,T]$.

\paragraph{Reachability Objective}
To navigate to a desired goal state, we define the cost-to-go function by specifying an appropriate cost on the state and control, and solving for the optimal cost-to-go using the following function:
\begin{equation}
    J(\mathbf{x}_t, \mathbf{u}_t)  = \|\mathbf{x}_{\text{goal}} - f(\mathbf{x}_{t}, \mathbf{u}_t)\|= \|\mathbf{x}_{\text{goal}} - \mathbf{x}_{t+1}\|,
\end{equation}
where $\|\cdot\|$ is Euclidean norm between a predicted state $\mathbf{x}_{t+1} = f(\mathbf{x}_{t}, \mathbf{u}_t) $  and the goal state $\mathbf{x}_{\text{goal}} $.

To reduce the computational effort in evaluating the cost-to-go function over a short-horizon trajectory, we consider a subset of possible trajectories where the same control is applied, i.e., $\mathbf{u}_\tau = \mathbf{u}_t$ for $\tau = t,\ldots,t+\Delta$. The future cost incurred by being in a given state with $\Delta$ steps to go is then
\begin{equation}
    \mathcal{J}(\mathbf{x}_{[t:t+\Delta+1]}, \mathbf{u}_{[t:t+\Delta]}) = \sum_{\tau=t}^{t+\Delta} J(\mathbf{x}_\tau,\mathbf{u}_t) + J_F(\mathbf{x}_{t+\Delta+1}),
    \label{eqn:cost_to_go_one_control}
\end{equation}
where $J(\mathbf{x}_\tau,\mathbf{u}_t)$ is the stage cost and $J_F(\mathbf{x}_{t+\Delta+1})$ is the terminal cost. This serves as a proxy for the original cost function in the optimization problem \eqref{eqn:optimization_problem}, reducing the computational cost of evaluating the entire path.

Furthermore, using~\eqref{eqn:motion_model}, $\mathbf{x}_t$, and $\mathbf{u}_t$ to rewrite $\{\mathbf{x}_{t+\tau} \text{ for } \tau = 1,\ldots,\Delta\}$, it is possible to show that $\mathcal{J}$ depends only on $\mathbf{x}_t$ and $\mathbf{u}_t$ (with $J$, $J_F$, $f$, and $\Delta$ fixed). For that reason, we will use the shorthand $\mathcal{J}(\mathbf{x}_{[t:t+\Delta+1]}, \mathbf{u}_{[t:t+\Delta]})\coloneqq \mathcal{J}(\mathbf{x}_t, \mathbf{u}_t)$ from now on.

Note that like most local planners~\cite{fox1997dynamic,snape2011hybrid,williams2018information,mann2024control}, we assume the cost-to-go function in Eqn.~\eqref{eqn:cost_to_go_one_control} is admissible for a given planning horizon. 


\paragraph{Safety Constraints}
In the context of collision avoidance, the constraint functions $c_k$ can be defined to ensure that the robot does not collide with any obstacles. Let $\mathcal{O}_k$ represent the $k$-th obstacle in the environment. The collision constraint is defined as:
\begin{equation}
        c_k(\mathbf{x}_t) = r_{\text{safe}} - d(\mathbf{x}_t, \mathcal{O}_k),
        \label{eqn:constraints}
    \end{equation}
where $d(\mathbf{x}_t, \mathcal{O}_k)$ is the distance between the robot's state $\mathbf{x}_t$ and the obstacle $\mathcal{O}_k$, and $r_{\text{safe}}$ is a safety radius that ensures a buffer zone around the obstacle to prevent collisions.

\paragraph{Approximating the Objective and Constraint Functions}
We recall that at each timestep $t$, the function described in Equation~\eqref{eqn:cost_to_go_one_control} is optimized to find the next optimal control $\mathbf{u}^\star_t$. Since $\mathcal{J}$ may become expensive to compute, we utilize BO to guide this search efficiently. Let $  \mathcal{J}(\mathbf{u}) = \mathcal{J}(\mathbf{x}_t, \mathbf{u}) $, and $ c_k(\mathbf{u}) =c_k(f(\mathbf{x}_t, \mathbf{u}))$. To compute $\mathbf{u}^\star_t$, BO models the objective function and the constraints as GPs over the domain $ V_d $. We train the GPs by utilizing $p$ evaluations of the  objective function and constraints within $V_d$, constructing $K+1$ datasets: 

\begin{equation}
    \begin{aligned}
     D_\mathcal{J} &= \Big\{\big(\mathbf{u}^{(i)}_t, \mathcal{J}(\mathbf{x}_t, \mathbf{u}^{(i)}_t)\big)\Big\}_{i=1}^p , \\  D_{c_k} &= \Big\{\big(\mathbf{u}^{(i)}_t, c_k(f(\mathbf{x}_t, \mathbf{u}^{(i)}_t))\big)\Big\}_{i=1}^p, \; k = 1,\ldots,K.
    \end{aligned}
\end{equation}

Thus, the objective function and the constraints are approximated as GP realizations:

\begin{equation}
    \begin{aligned}
        \mathcal{J}(\mathbf{u}) &\sim \mathcal{GP}(\mu_\mathcal{J}(\mathbf{u}), \mathcal{K}_\mathcal{J}(\mathbf{u}, \mathbf{u}')) \\
        c_k(\mathbf{u}) &\sim \mathcal{GP}(\mu_{c_k}(\mathbf{u}), \mathcal{K}_{c_k}(\mathbf{u}, \mathbf{u}')), \; k = 1,\ldots, K.
    \end{aligned}
\end{equation}

Aside from approximating the objective function using GPs, BO samples promising points using an acquisition function that quantifies the probability of each sample to have a high value when evaluated by the objective function. In unconstrained scenarios, the standard Expected Improvement acquisition function $\text{EI}(\mathbf{u})$ is used as:

\begin{equation}
    \text{EI}(\mathbf{u}) = (y^- -\mu_\mathcal{J}( \mathbf{u}))\Phi(z) + \sigma_\mathcal{J}(\mathbf{u})\phi(z),
    \label{eqn:ei}
\end{equation}

where $y^-$ is the current best observed objective value;
$\mu_\mathcal{J}(\mathbf{u})$ is the posterior mean at $\mathbf{u}$;
$\sigma_\mathcal{J}(\mathbf{u}) = \sqrt{\mathcal{K}_\mathcal{J}(\mathbf{u}, \mathbf{u})}$ is the posterior standard deviation at $\mathbf{u}$;
$z = \frac{y^- -\mu_\mathcal{J}(\mathbf{u})}{\sigma_\mathcal{J}(\mathbf{u})}$.
$\Phi$ and $\phi$ are the CDF and the PDF of the standard normal distribution, respectively. This function provides a reasonable trade-off between exploration and exploitation.



{
Building upon the Constrained Expected Improvement (CEI) acquisition function~\cite{gelbart2014bayesian}, we developed a modified EI that incorporates control input constraints to achieve a balance between expected objective function improvement and constraint satisfaction probability as:
}

\begin{equation}
    \text{CEI}(\mathbf{u}) = \text{EI}(\mathbf{u})  P(\text{feasible} \mid \mathbf{u})\label{eqn:cei},
\end{equation}

where 
\begin{equation}
    P(\text{feasible} \mid \mathbf{u}) = \prod_{k=1}^K P(c_k(\mathbf{u}) \leq 0 \mid D_{c_k})
    \label{eqn:fesaible}
\end{equation}
is the probability that all constraints are satisfied at $\mathbf{u}$ assuming the constraints are conditionally independent given $\mathbf{u}$.

Furthermore, each of these probabilities has a closed form as 
\begin{equation}
    P(c_k(\mathbf{u}) \leq 0 \mid D_{c_k}) = \Phi\left( \frac{-\mu_{c_k}(\mathbf{u})}{\sigma_{c_k}(\mathbf{u})} \right),\label{eqn:cei_prob}
\end{equation}
where $\mu_{c_k}(\mathbf{u})$ and $\sigma_{c_k}(\mathbf{u})$ are the posterior mean and standard deviation of the GP for the $k$-th constraint at $\mathbf{u}$.


Finally, the optimal control input $\mathbf{u}^\star_t$ to evaluate is chosen by maximizing the CEI:

\begin{equation}
    \mathbf{u}^\star_t = \argmax_{\mathbf{u} \in V_d}\; \text{CEI}(\mathbf{u}).\label{eqn:optimal}
\end{equation}

Equation~\eqref{eqn:optimal} ensures that the optimization process not only seeks to improve the objective function but also respects the constraints by considering the probability of feasibility.

\subsection{Receding Horizon Planning and Control}
We employ a receding horizon optimization setup, where control actions are determined iteratively to guide the robot toward its goal. Algorithm~\ref{alg:BOW_Planning} details the BOW planning procedures. The process begins by initializing an empty final trajectory, which will store the sequence of feasible states the robot follows. The algorithm then enters a continuous loop that persists until the robot reaches the goal, defined by a specified goal radius $r$.

\begin{algorithm}[h!]
    \caption{BOW Planning and Control}
    \label{alg:BOW_Planning}
    \begin{algorithmic}[1]
        \STATE \textbf{Input:} Initial state $\mathbf{x}_0$, Goal state $\mathbf{x}_{\text{goal}}$, Parameters $\xi$, Goal distance $r$
        \STATE \textbf{Output:} Final trajectory $\mathbf{x}_{[0:t]}$, Optimal controls $\mathbf{u}_{[0:t-1]}^\star$
        \STATE Initialize $t \gets 0$, $\mathbf{x}_t \gets \mathbf{x}_0$
        \WHILE{$\|\mathbf{x}_t - \mathbf{x}_{\text{goal}}\| \leq r$}
            \STATE Define planning window $V_d \subseteq \mathcal{U}$ using $f$ and $\xi$
            \STATE Sample $p$ controls $\{\mathbf{u}_t^{(i)}\}_{i=1}^p \subseteq V_d$
            \STATE Initialize datasets: $D_{\mathcal{J}}\gets \emptyset$ and $D_{c_k} \gets \emptyset$, for $k=1,\ldots,K$.
            \FOR{$i = 1,\ldots, p$}
                \STATE Compute $\mathcal{J}(\mathbf{x}_t, \mathbf{u}_t^{(i)})$ using Equation~\eqref{eqn:cost_to_go_one_control}
                \STATE Compute constraint values $c_k(f(\mathbf{x}_t,\mathbf{u}_t^{(i)})) $ for each $k=1,\ldots, K$ using Eqaution~\eqref{eqn:constraints}
                \STATE Update datasets: $D_\mathcal{J} \gets D_\mathcal{J} \cup \{(\mathbf{u}_t^{(i)}, \mathcal{J}(\mathbf{x}_t, \mathbf{u}_t^{(i)}))\}$ and $D_{c_k} \gets D_{c_k} \cup \{\mathbf{u}_t^{(i)}, c_k(f(\mathbf{x}_t,\mathbf{u}_t^{(i)}))\}$ 
            \ENDFOR
            \STATE Train the $K+1$ GPs using $D_\mathcal{J}$ and $D_{c_k}$
            \STATE Compute CEI$(\mathbf{u})$ using Equations~\eqref{eqn:ei},\eqref{eqn:cei} and \eqref{eqn:cei_prob}
            \STATE Choose $\mathbf{u}_t^\star = \argmax_{\mathbf{u} \in V_d} \text{CEI}(\mathbf{u})$
            \STATE Generate trajectory $\hat{\mathbf{x}}_{[t:t+\Delta]} = (\hat{\mathbf{x}}_{t},\hat{\mathbf{x}}_{t+1} \dots, \mathbf{x}_{t+\Delta})$ using RK4 on $f$, $\mathbf{u}_t^\star$ and $\mathbf{x}_{t}= \hat{\mathbf{x}}_t$
            \STATE Apply control $\mathbf{u}_t^\star$ for $t+\ell$ and obtain trajectory $\hat{\mathbf{x}}_{[t:t+\ell]}$ whereas $\ell \leq \Delta$
            \IF{$c_k(\mathbf{x}_{\tau})\leq0$ for $\tau = t,\ldots,t+\ell$, $k=1,\ldots,K$}
                \STATE $\mathbf{x}_{t}\gets \hat{\mathbf{x}}_{t+\ell}$
                \STATE $t \gets t + \ell$
            \ENDIF
            
        \ENDWHILE
        \STATE \textbf{return} $\mathbf{x}_{[0:t]}$, $\mathbf{u}_{[0:t]}^\star$
    \end{algorithmic}
\end{algorithm}

At each iteration (line 4), the BOW Planner is instantiated with the current state $\mathbf{x}_t$, the goal $\mathbf{x}_{\text{goal}}$, a collision checker, and a set of parameters $\xi$. This planner leverages the CBO approach to compute the optimal control $\mathbf{u}^*_t$.

As shown in lines 5--15, $p$ control candidates are sampled from a local planning window $V_d$, such that $\{\mathbf{u}_t^{(i)}\}_{i=1}^p \subseteq V_d$. The cost and constraint satisfaction of each control $\mathbf{u}_t^{(i)}$ are evaluated using Equations~\eqref{eqn:cost_to_go_one_control}--\eqref{eqn:constraints}. The CEI acquisition function facilitates effective constraint handling by modeling both the objective and constraint functions via GPs, ensuring that the selected control satisfies feasibility conditions such as obstacle avoidance.

Using the optimal control $\mathbf{u}^*_t$, a receding horizon trajectory is generated via the motion model at line 16, using fourth-order Runge-Kutta (RK4) integration. This trajectory consists of a sequence of predicted states over the planning horizon. The algorithm then verifies whether this trajectory is collision-free using a collision checker.

If the trajectory is valid, the algorithm determines an index $\ell$ (line 17), which indicates how far into the planning horizon the same control $\mathbf{u}^\star_t$ can be applied. This allows the method to extend control application over multiple future steps, offering greater flexibility than traditional Model Predictive Control (MPC), which typically only applies the immediate control at $t+1$.



{
In practical implementation with finite samples, a single evaluation of the CEI in Eqn.~\eqref{eqn:optimal} generates solutions that formulate collision avoidance as soft chance constraints, thereby identifying candidate points that simultaneously enhance the objective function while maintaining high probability of safety compliance. The BOW planner employs a trajectory validation mechanism at line 18, which evaluates the complete predicted trajectory $\mathbf{x}_\tau$ across the time horizon $\tau = t, \ldots, t + \ell$. Upon successful validation where the trajectory satisfies all imposed constraints $c_k$ and maintains collision-free operation throughout the prediction horizon, the algorithm updates the current state $\mathbf{x}_t$ to the predicted state at index $\ell$, effectively advancing the planning window. Conversely, when the predicted trajectory fails to meet the feasibility criteria, the planner initiates an iterative sampling process for alternative control inputs, continuing until either a feasible trajectory is identified or the predetermined maximum iteration threshold is exceeded.
}

The loop proceeds by evaluating the distance between the updated state $\mathbf{x}_t$ and the goal $\mathbf{x}_{\text{goal}}$. If this distance is less than or equal to the threshold $r$ (line 20), the algorithm terminates, indicating that the robot has successfully reached its target. Upon termination, the algorithm returns a success flag and the final trajectory, representing the complete path from the initial state to the goal.


    \section{Algorithm Analysis}


\begin{lemma} 
    
    Let \( g: V_d \subseteq \mathcal{U} \to \mathbb{R} \) be a Lipschitz continuous function on a compact set \( V_d \). If a GP with an universal kernel (e.g., squared exponential kernel, see \cite{micchelli2006universal}) models \( g \), then the posterior mean \( \mu(\mathbf{u}) \) converges uniformly to \( g(\mathbf{u}) \) as the number of samples tends to infinity.
\end{lemma}
\begin{proof}[Sketch of the Proof]
    \cite{micchelli2006universal} shows that the set of functions that the squared exponential kernel generates is dense in $C(V_d)$. Moreover, \cite{vanzanten2008} shows that the error, expressed in some norm defined on $V_d$, converges to zero when the number of samples tends to infinity.

\end{proof}


\begin{theorem}
    Let \( V_d \subseteq \mathcal{U} \) be a compact planning window defined by the robot's velocity and acceleration limits. Suppose the objective function \( \mathcal{J}(\mathbf{u})\) and constraint functions \( c_k(\mathbf{u}) \) for \( k = 1, \dots, K \) are Lipschitz continuous over \( V_d \). Assume the system dynamics \( f \) are Lipschitz continuous, and the GP models \( \mathcal{J} \) and \( c_k \) use kernels such as the squared exponential kernel. Then, as the number of samples tends to infinity, the control input \( \mathbf{u}^\star_t \) selected by maximizing the CEI converges to the true optimal control input \( \mathbf{u}_{t,\text{opt}} = \argmin_{\mathbf{u} \in V_d} \mathcal{J}(\mathbf{u}) \) within \( V_d \) while satisfying all constraints \( c_k(\mathbf{u}) \leq 0 \).
\end{theorem}

\begin{proof}
    We prove this theorem in five steps, leveraging properties of GPs and BO under the given assumptions.

    \subsubsection*{Step 1: GP Approximation and Uniform Convergence}
    Since \( \mathcal{J}(\mathbf{u}) \) and \( c_k(\mathbf{u}) \) are Lipschitz continuous on the compact set \( V_d \), Lemma 1 guarantees that their GP models converge to the functions as the number of samples increases. 

    \subsubsection*{Step 2: Convergence of the Probability of Feasibility}
    For each constraint \( c_k(\mathbf{u}) \), the probability of feasibility is modeled as
    \begin{equation}
        P(c_k(\mathbf{u}) \leq 0 \mid D_{c_k}) = \Phi\left( \frac{-\mu_{c_k}(\mathbf{u})}{\sigma_{c_k}(\mathbf{u})} \right),
    \end{equation}
    recalling that \( \Phi \) is the standard normal CDF, and \( D_{c_k} \) is the dataset of constraint observations. As the number of samples tends to infinity, \( \mu_{c_k}(\mathbf{u}) \to c_k(\mathbf{u}) \) and \( \sigma_{c_k}(\mathbf{u}) \to 0 \). Thus,
    \[
        P(c_k(\mathbf{u}) \leq 0 \mid D_{c_k}) \to
        \begin{cases}
            1 & \text{if } c_k(\mathbf{u}) < 0, \\
            0.5 & \text{if } c_k(\mathbf{u}) = 0, \\
            0 & \text{if } c_k(\mathbf{u}) > 0.
        \end{cases}
    \]

    In the feasible region where \( c_k(\mathbf{u}) \leq 0 \) for all \( k \) (assuming strict feasibility \( c_k(\mathbf{u}) < 0 \) for simplicity), \( P(\text{feasible} \mid \mathbf{u}) \to 1 \) (recall Equation~\eqref{eqn:fesaible}). For infeasible \( \mathbf{u} \), where \( c_k(\mathbf{u}) > 0 \) for some \( k \), \( P(\text{feasible} \mid \mathbf{u}) \to 0 \).

    \subsubsection*{Step 3: Convergence of the EI process}
    \cite{VAZQUEZ20103088} shows that an EI process produces a dense set of samples that converges to $\argmin_{\mathbf{u} \in V_d} \mathcal{J}(\mathbf{u})$ when the number of samples tends to infinity.

    \subsubsection*{Step 4: Behavior of the CEI function}
    Recall that CEI combines EI with feasibility:
    \begin{equation}
        \text{CEI}(\mathbf{u}) = \text{EI}(\mathbf{u})  P(\text{feasible} \mid \mathbf{u}).
    \end{equation}
    As the number of samples tends to infinity, we combine the results of the two previous steps, then, the CEI acquisition function becomes

    \begin{equation}
        \text{CEI}(\mathbf{u}) = \text{EI}(\mathbf{u})\prod_{k=1}^K\mathbf{1}_{\{c_k(\mathbf{u})\leq 0\}}.
        \label{eqn:cei_restricted}
    \end{equation}

    \subsubsection*{Step 5: Convergence to the Optimal Solution}
    Direclty from Equation~\eqref{eqn:cei_restricted}, we have that
        \begin{multline}
        \mathbf{u}^\star_t = \argmax_{\mathbf{u} \in V_d} \text{CEI}(\mathbf{u})=
        \argmax_{\mathbf{u} \in V_d}\text{EI}(\mathbf{u})\prod_{k=1}^K\mathbf{1}_{\{c_k(\mathbf{u})\leq 0\}}\\ = \argmax_{\mathbf{u} \in V_d, \,c_k(\mathbf{u})\leq 0, \, \forall k }\mathcal{J}(\mathbf{u})=\mathbf{u}_{\text{opt}},
        \end{multline}
    completing the proof.
\end{proof}

\subsection{Computational Complexity of BOW Planner}
The BOW Planning algorithm iteratively computes an optimal control \( \mathbf{u}^\star_t \) using BO with GPs to generate collision-free trajectories over \( \psi \) iterations. With a single optimization iteration (\( \psi = 1 \)) and a constant number of GP samples \( p \), the complexity per iteration includes:  
\begin{itemize}
    \item \textbf{Control Optimization}: GP update costs \( O(p^3) \), a constant since {\( p \)} is fixed.  
    \item \textbf{Trajectory Generation}: RK4 integration over horizon \( \Delta \) with state dimension \( n \) costs \( O(\Delta n) \).  
    \item \textbf{Collision Checking}: 
    Collision checking of $\Delta$ points using a Bounding Volume Hierarchy (BVH) model achieves a per-query complexity of \( O(\log C) \), resulting in a total complexity of \( O(\Delta \log C) \), where \( C \) represents the number of geometric primitives in the hierarchy.
\end{itemize}
Total complexity is:
{\[
O(\psi  (p^3 + \Delta  (n + \log  C))) = O(\psi  \Delta  (n + \log  C))
\]}
where \( \psi \) is the number of iterations (problem-dependent), \( \Delta \) is the horizon length, \( n \) is the state dimension, and \( C \) is the collision checking complexity. 
    \section{Experiments and Results}

{
All experiments were conducted on a desktop computer running Ubuntu 22.04 LTS, equipped with an AMD Ryzen\textsuperscript{TM} 9 7950X processor, an NVIDIA GTX 4070 GPU, and 32\,GB of RAM. Robot localization was performed using a VICON motion capture system operating at 120\,Hz. The BOW Planner was configured with a maximum speed of 1.0\,m/s, a minimum speed of 0.0\,m/s, a maximum acceleration of 0.5\,m/s\textsuperscript{2}, a maximum yaw rate of 0.6981\,rad/s, and a maximum yaw acceleration of 2.0472\,rad/s\textsuperscript{2}. The $\ell$ was set to $7$. The planning operated with a time step of 0.1\,s and a prediction horizon of 3.0\,s. For trajectory optimization, we used a Squared Exponential ARD kernel (Radial Basis Function) within the Bayesian Optimization framework.
}

\vspace{-10pt}

\subsection{Sample Efficiency}
The sample efficiency property of the BOW planner is further illustrated in Fig.~\ref{fig:ablation_study}.
\begin{figure}[t]
    \centering
    \begin{subfigure}{0.20\textwidth}
        \centering
        \includegraphics[width=\textwidth, trim= 14.5cm 9cm 7.5cm 5cm, clip]{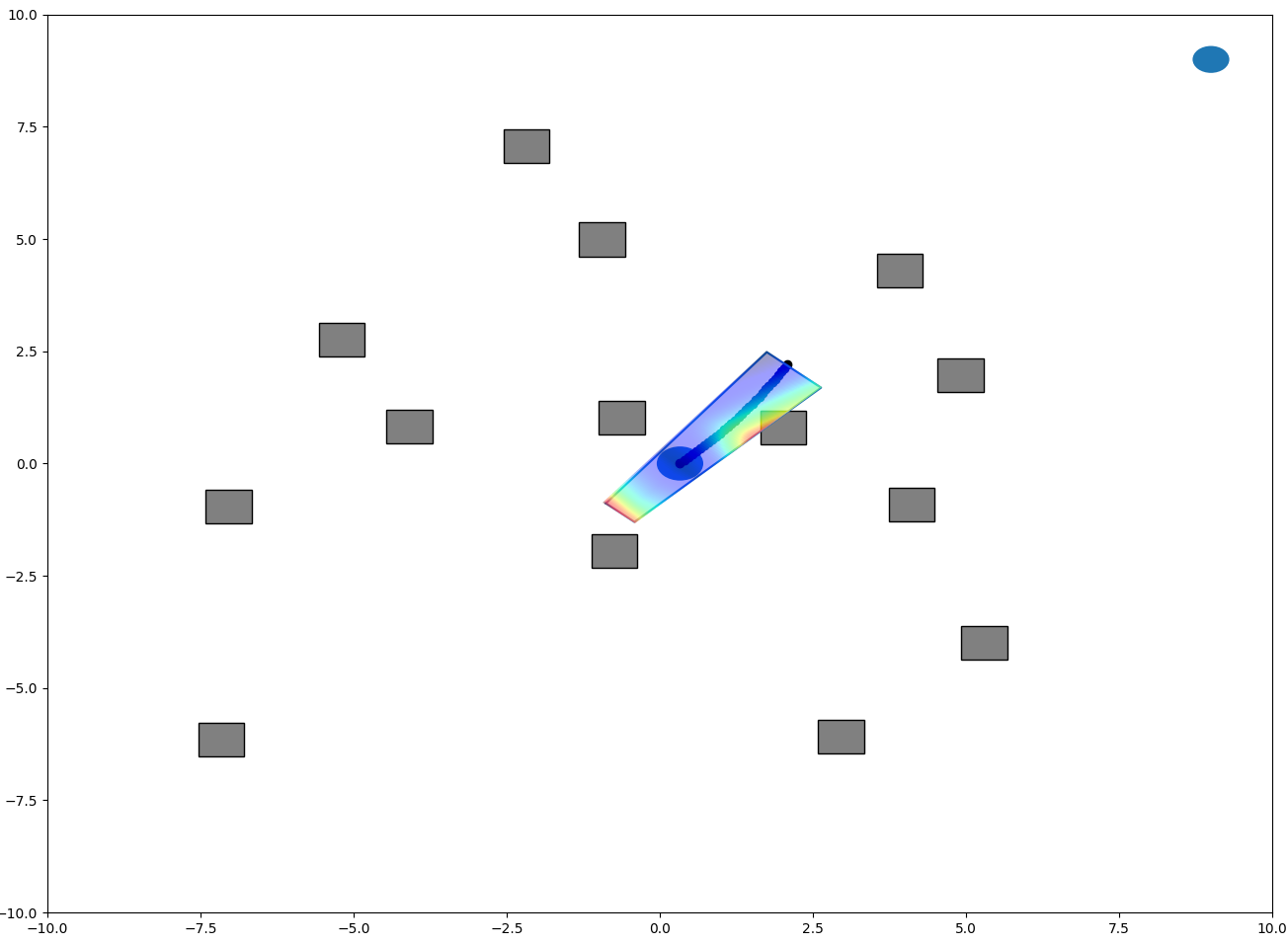}
        \caption{}
    \end{subfigure}
    ~
    \begin{subfigure}{0.20\textwidth}
        \centering
        \includegraphics[height=2.3cm]{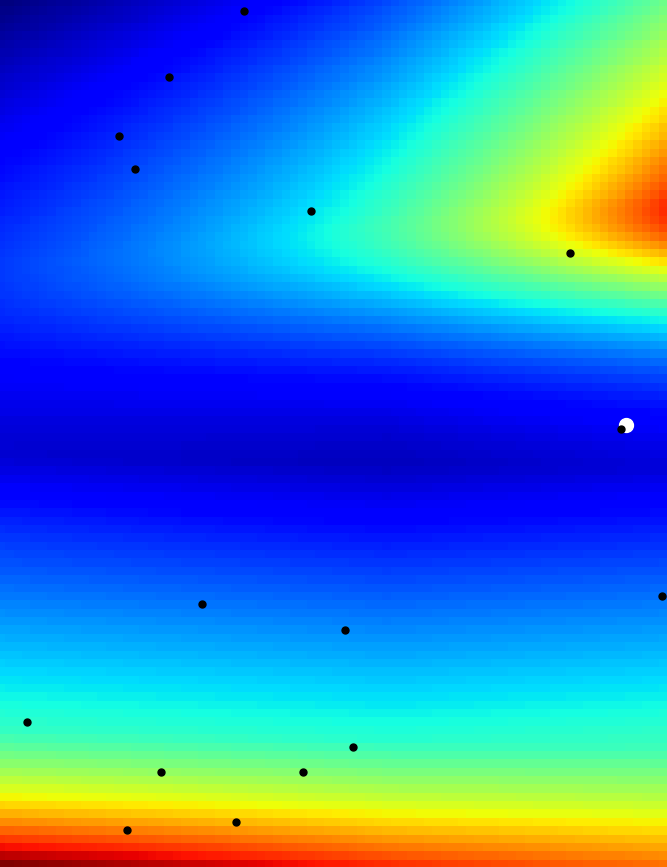}
        \caption{}
    \end{subfigure}
    \caption{\textbf{Sample Efficiency:} The DWA performs a grid search to assess the cost function (depicted by the colored map) using the dynamic window, while the BOW employs constrained Bayesian optimization to efficiently determine the optimal control (represented by white dots) with {$15$ samples only}  (indicated by black dots).}
  \label{fig:ablation_study}
  \vspace{-5pt}
  \end{figure}
Given a dynamic window, the DWA planner performs a grid search to compute the optimal solution over the cost function, depicted as a colored map. In contrast, the MPPI planner uniformly samples the cost function from both feasible (blue-colored region) and infeasible (obstacle regions) spaces.
Unlike other methods, the BOW planner efficiently computes the optimal solution (white dot) by evaluating samples from the feasible space more than the infeasible space, resulting in faster runtime efficiency. This demonstrates the superior sample efficiency of the BOW planner.
Fig.~\ref{fig:ablation_study} (b) illustrates the sampling behavior of the BOW planner. In this representation, the black dots indicate the samples taken by the BOW planner, while the white dot marks the optimal sample, and the colored map represents the cost map.
\subsection{Unmanned Ground Vehicle Experiments}
We implemented the BOW planner for a non-holonomic differential wheel UGV(Unmanned Ground Vehicle) navigating through a 2-D space containing obstacles. We performed simulations for UGVs across six different environments, as shown in Fig. \ref{fig:environments}. The UGV’s state is defined by a 5-D vector $\mathbf{x} = (x, y, \theta, v, \omega)$, encompassing its position $(x, y)$, orientation $\theta$, linear velocity $v$ and angular velocity $\omega$. The control input is represented as a 2-D vector $\mathbf{u} = (v_c, \omega_c)$, which includes the linear velocity $v_c$ and angular velocity $\omega_c$ with respect to the body frame.
\begin{figure}[h!]
  \centering
    \centering
      \begin{subfigure}[b]{\linewidth}
        \includegraphics[width=\linewidth, height=3.4cm]{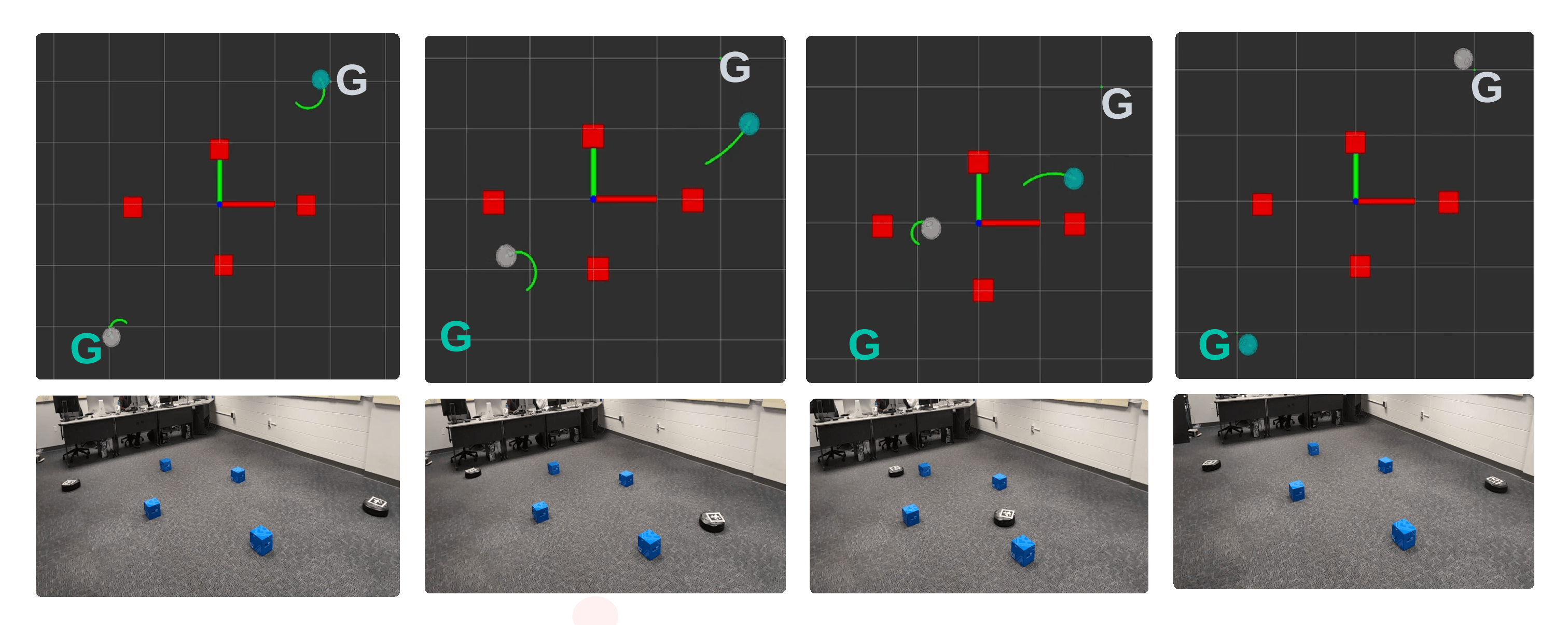}
        \centering
        \caption{}
        \label{fig:Dynamic_obstacles}
      \end{subfigure}
      \hspace{1pt}
    \centering
      \begin{subfigure}[b]{\linewidth}
        \includegraphics[width=0.97\linewidth, height=2.0cm]{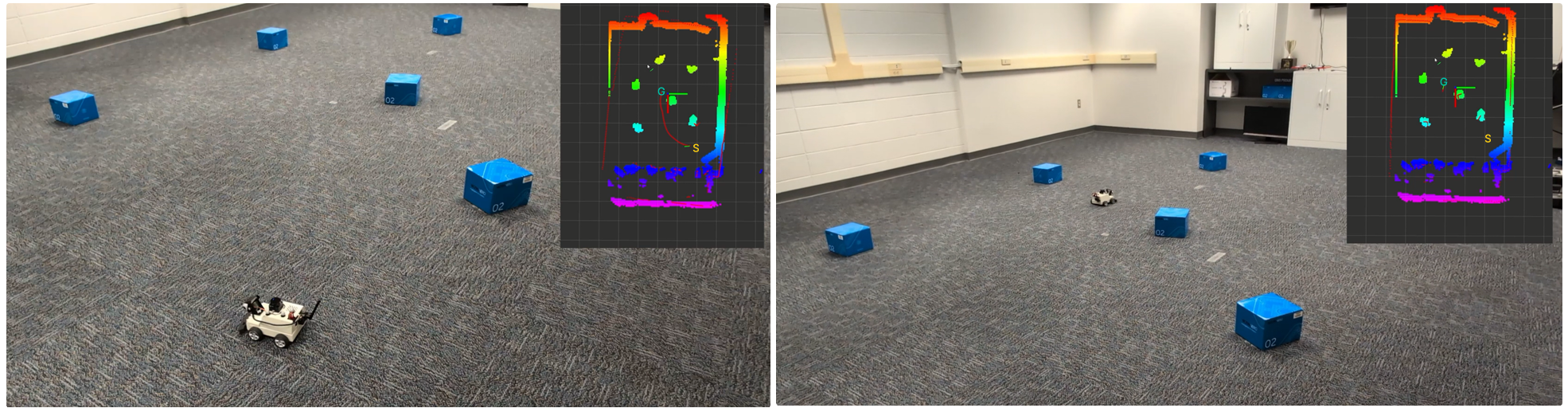}
        \centering
        \caption{}
        \label{fig:occupancy_map}
      \end{subfigure}
  \caption{
   The BOW planner is applied to UGVs for navigating environments with dynamic obstacles, leveraging its rapid planning capabilities~\ref{fig:Dynamic_obstacles}. The planner also has the ability to avoid obstacles using onboard LiDAR and to follow the planned trajectory~\ref{fig:occupancy_map}.}
   
\label{fig:s&d_obs}
\vspace{-5pt}
\end{figure}
\begin{figure*}[t]
    \centering
    \vspace{-20pt}
    \includegraphics[width=0.16\linewidth,viewport=0pt 0pt 900pt 900pt,clip]{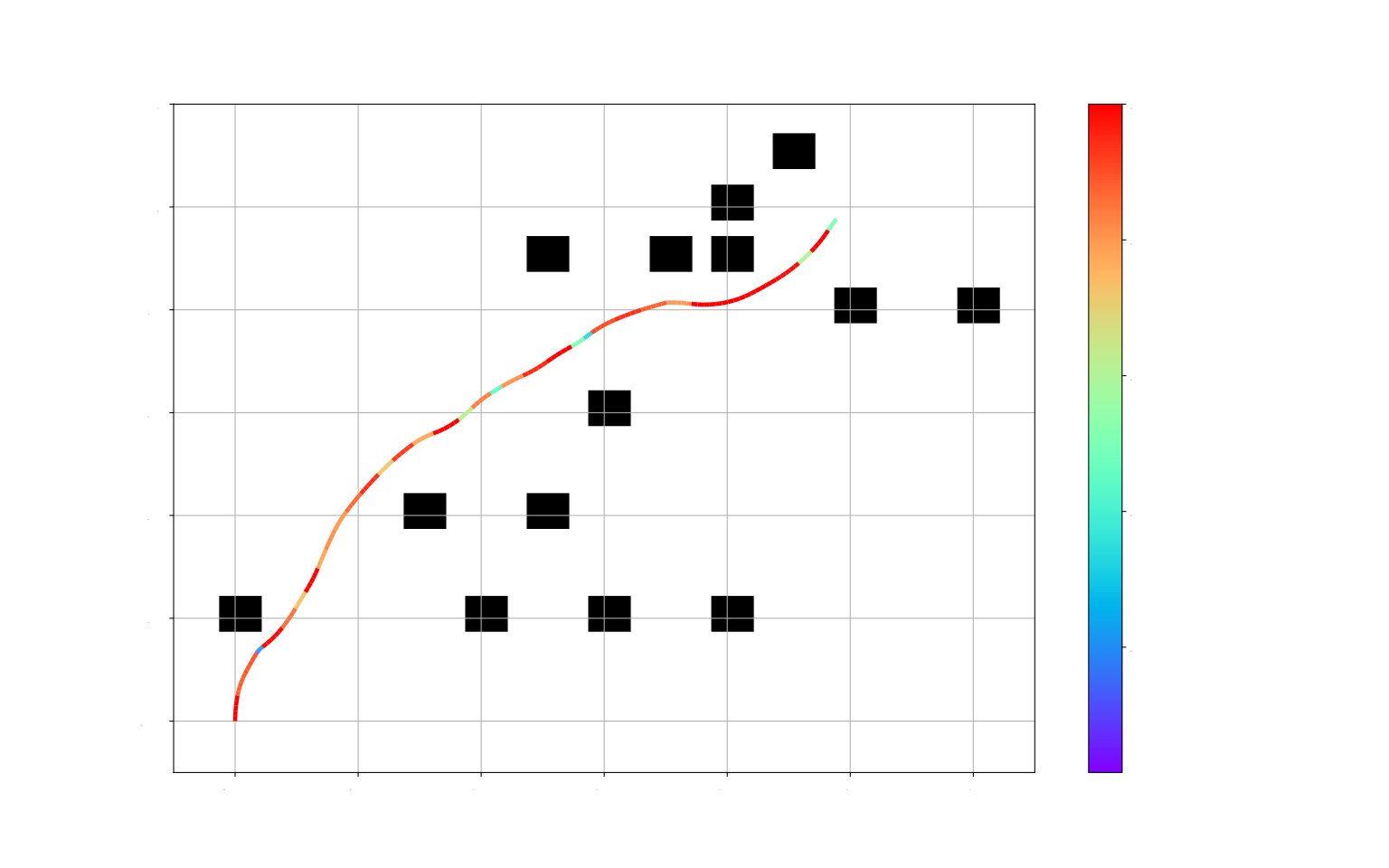} 
    \includegraphics[width=0.16\linewidth,viewport=0pt 0pt 900pt 900pt,clip]{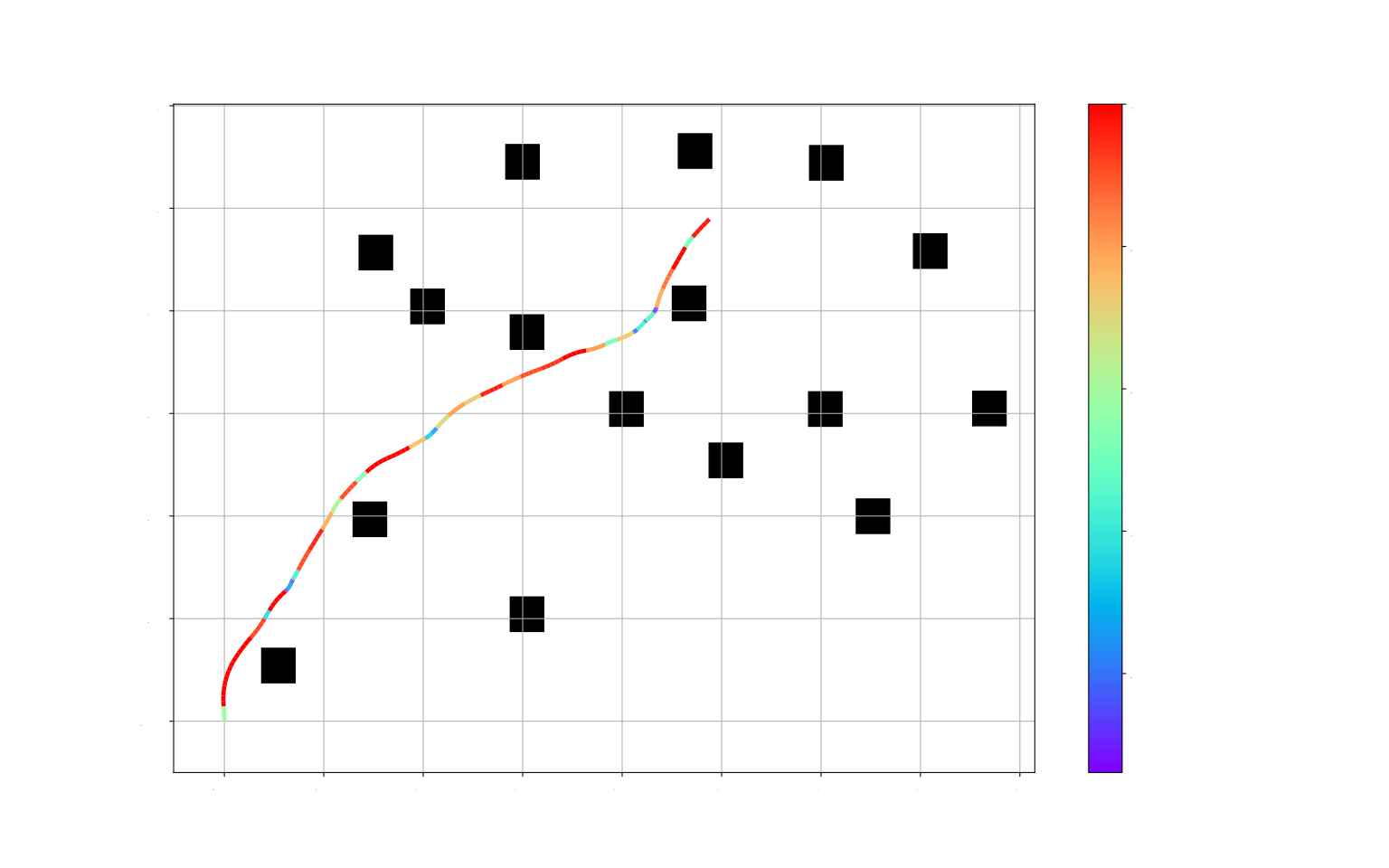} 
    \includegraphics[width=0.16\linewidth,viewport=0pt 0pt 900pt 900pt,clip]{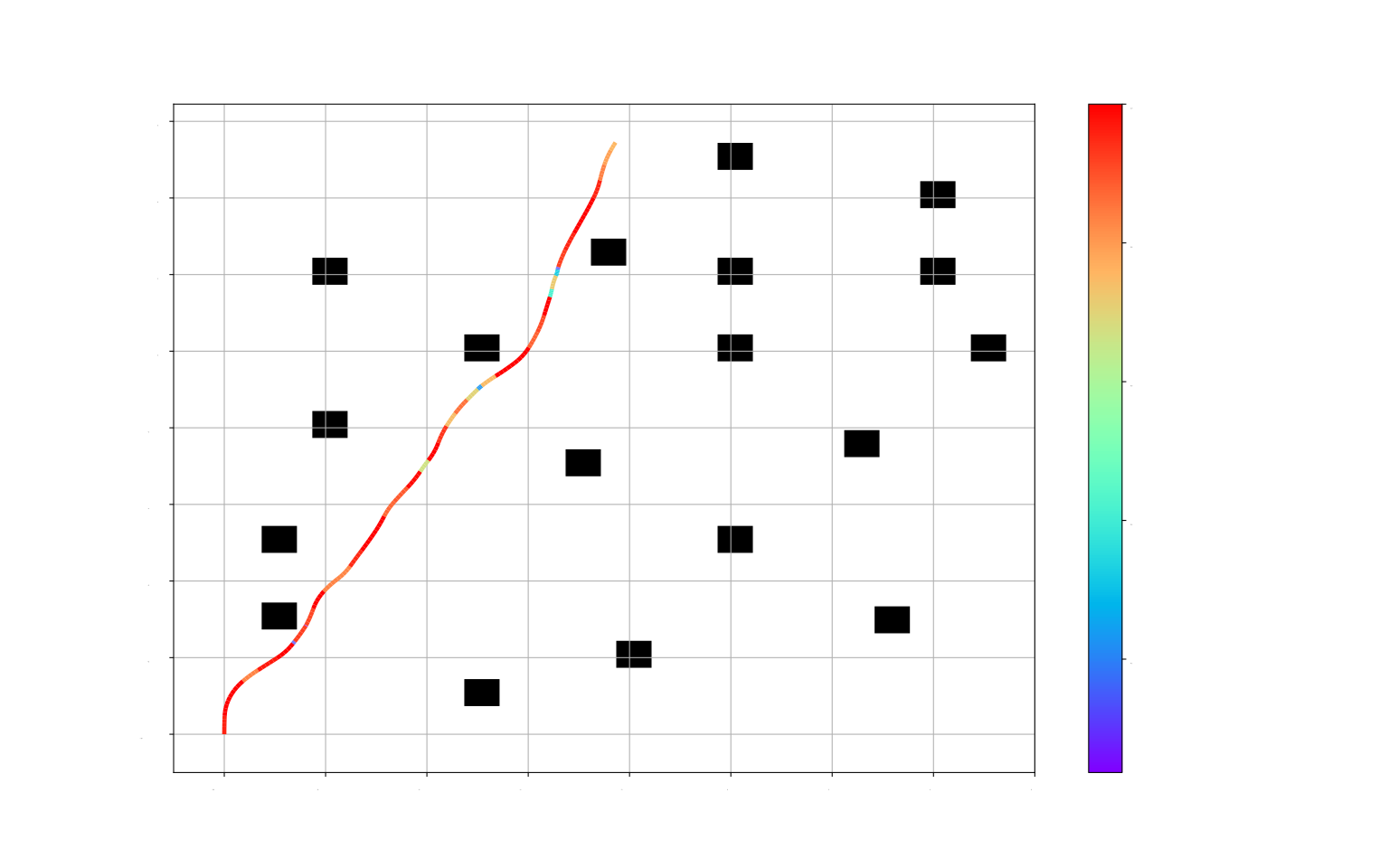} 
    \includegraphics[width=0.16\linewidth,viewport=0pt 0pt 900pt 900pt,clip]{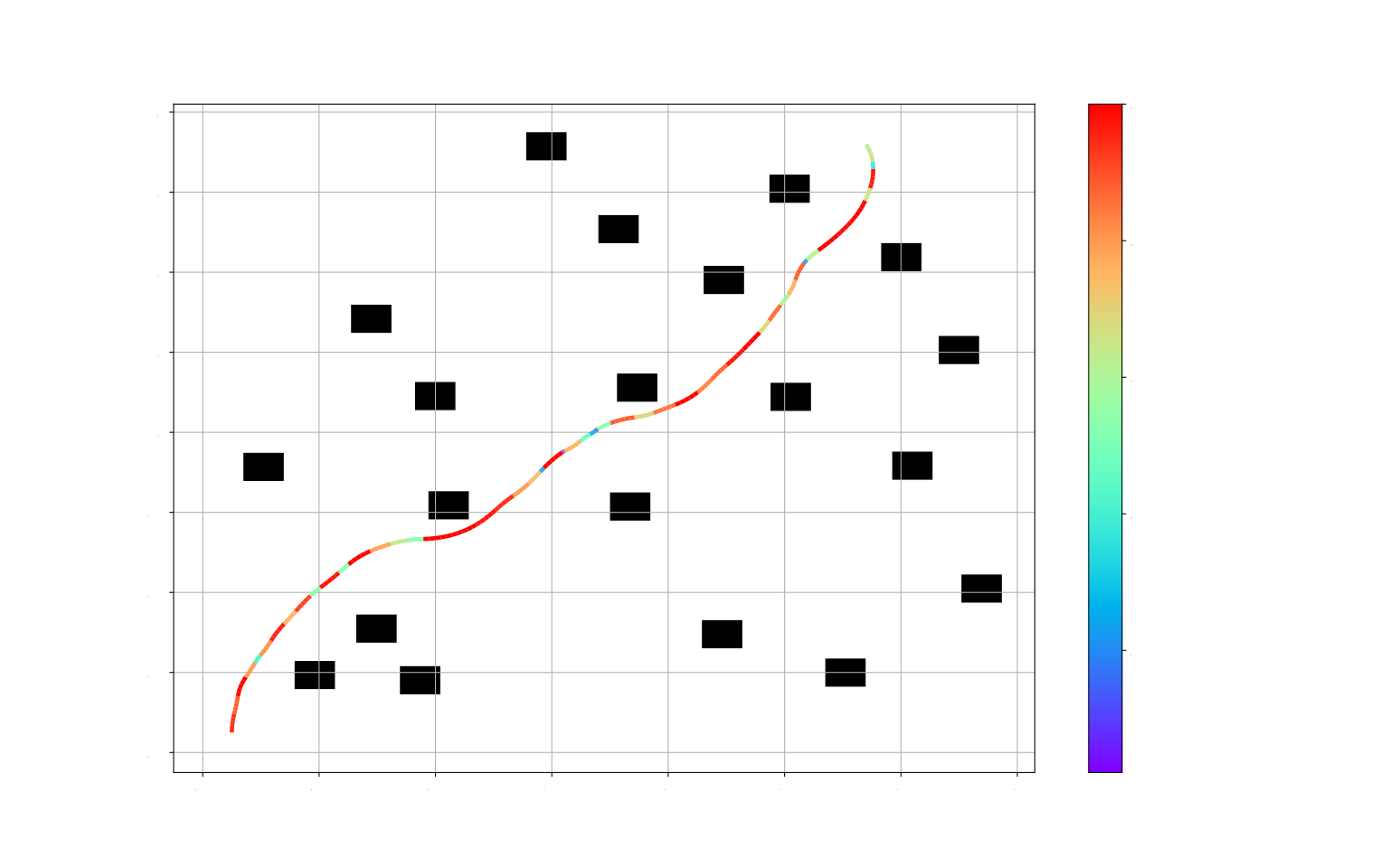} 
    \includegraphics[width=0.16\linewidth,viewport=0pt 0pt 900pt 900pt,clip]{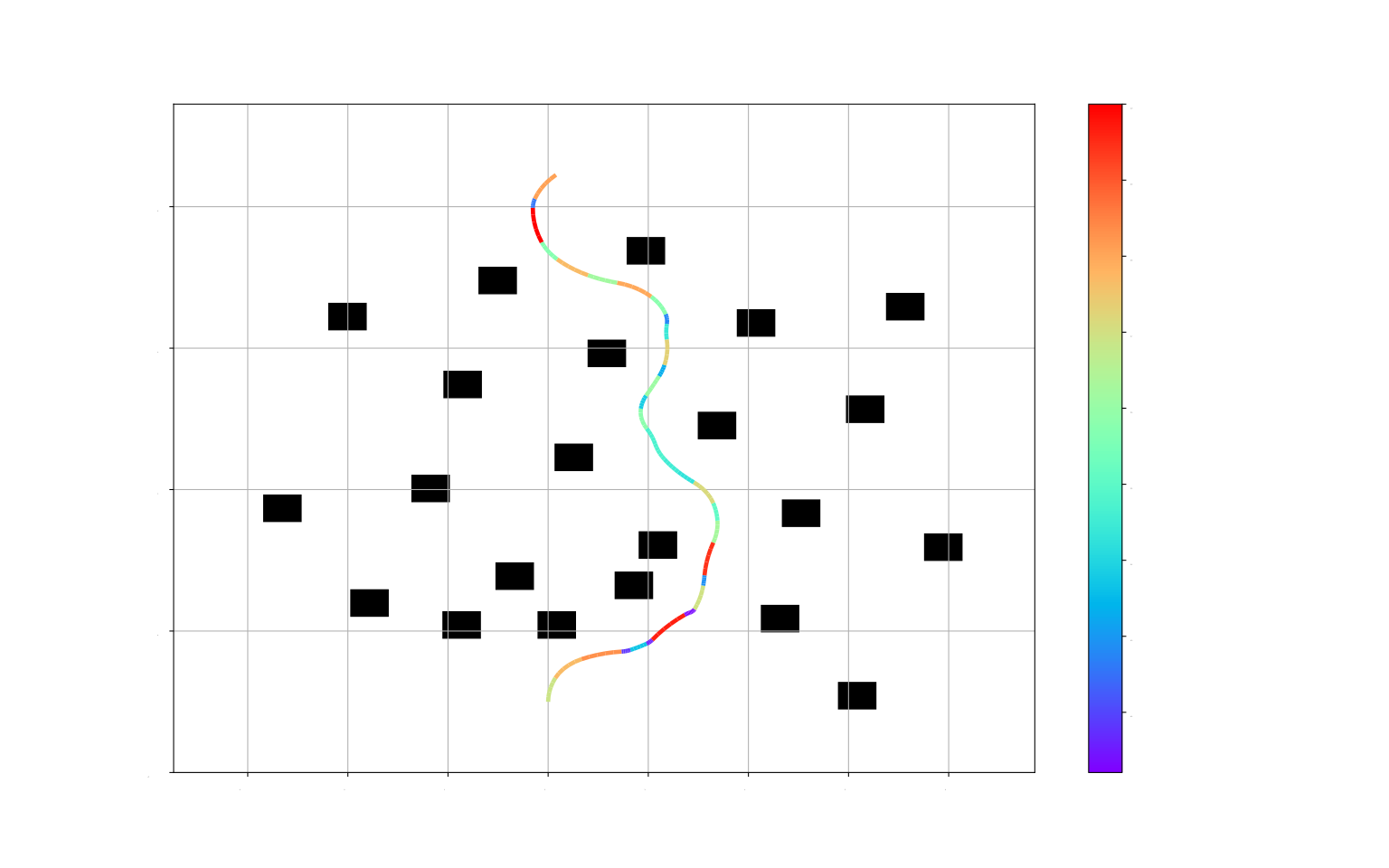} 
    \includegraphics[width=0.16\linewidth,viewport=0pt 0pt 900pt 900pt,clip]{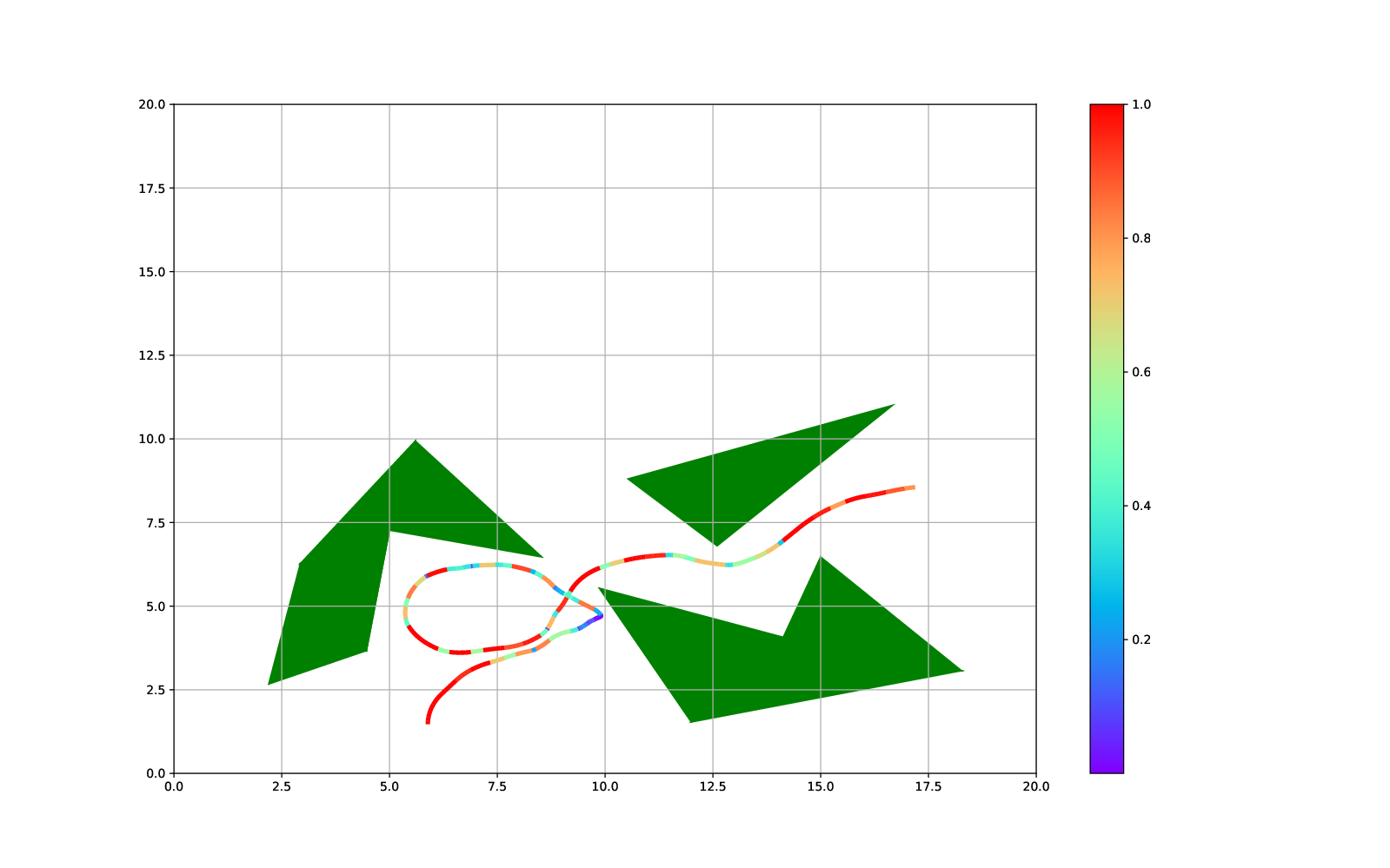} 
    
    \caption{{Six simulated 20m$\times$20m environments (Environments 1-6, from left to right) where black squares represent obstacles in the first five, and the sixth environment features non-convex obstacles composed of clustered green triangles. The curved trajectories display a blue-to-red gradient indicating increasing linear velocity from low to high.}}
    \label{fig:environments}
    \vspace{-20pt}
\end{figure*}

{The faster approximation of optimal control enables the BOW planner to scale effectively for distributed robot navigation problems. Fig.~\ref{fig:Dynamic_obstacles} demonstrates this capability in a two-UGV scenario where each robot navigates toward its respective goal while treating the other as a dynamic obstacle. The gray UGV targets the gray ``G" while the cyan UGV aims for the cyan ``G". The color coding helps differentiate the robots' paths as they adjust their trajectories to avoid collisions during navigation. For this experiment, pose information was shared between robots, with the planner running on a desktop computer that sent commands to the UGVs.
Fig.~\ref{fig:occupancy_map} illustrates the BOW planner's implementation on a unicycle UGV navigating an environment with static obstacles. The robot uses an onboard laser scanner to detect obstacles and generate an occupancy map, enabling it to avoid obstacles while tracking the trajectory planned by the BOW planner. The UGV's goal location is indicated by the cyan ``G", and the figure shows the planned path as the robot navigates around obstacles to reach its destination.}
\begin{figure}[h!]
\centering
\begin{subfigure}[t]{0.21\textwidth}
\centering
\includegraphics[width=\textwidth, trim={5.0cm 2.0cm 10.0cm 3.0cm}, clip]{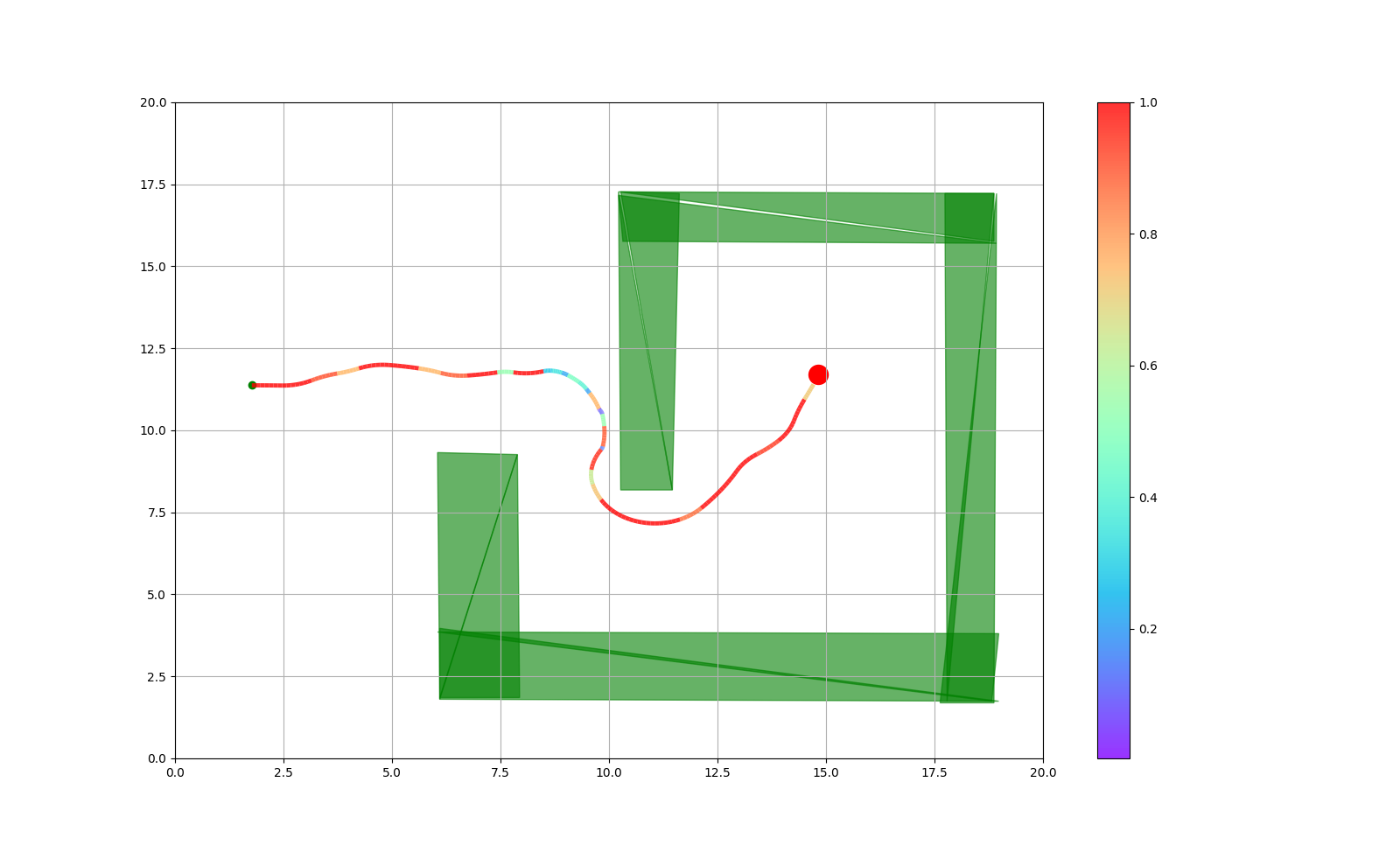}
\caption{Bugtrap Environment}
\label{fig:bugtrap}
\end{subfigure}
\hfill
\begin{subfigure}[t]{0.21\textwidth}
\centering
\includegraphics[width=\textwidth, trim={5.0cm 2.0cm 10.0cm 3.0cm}, clip]{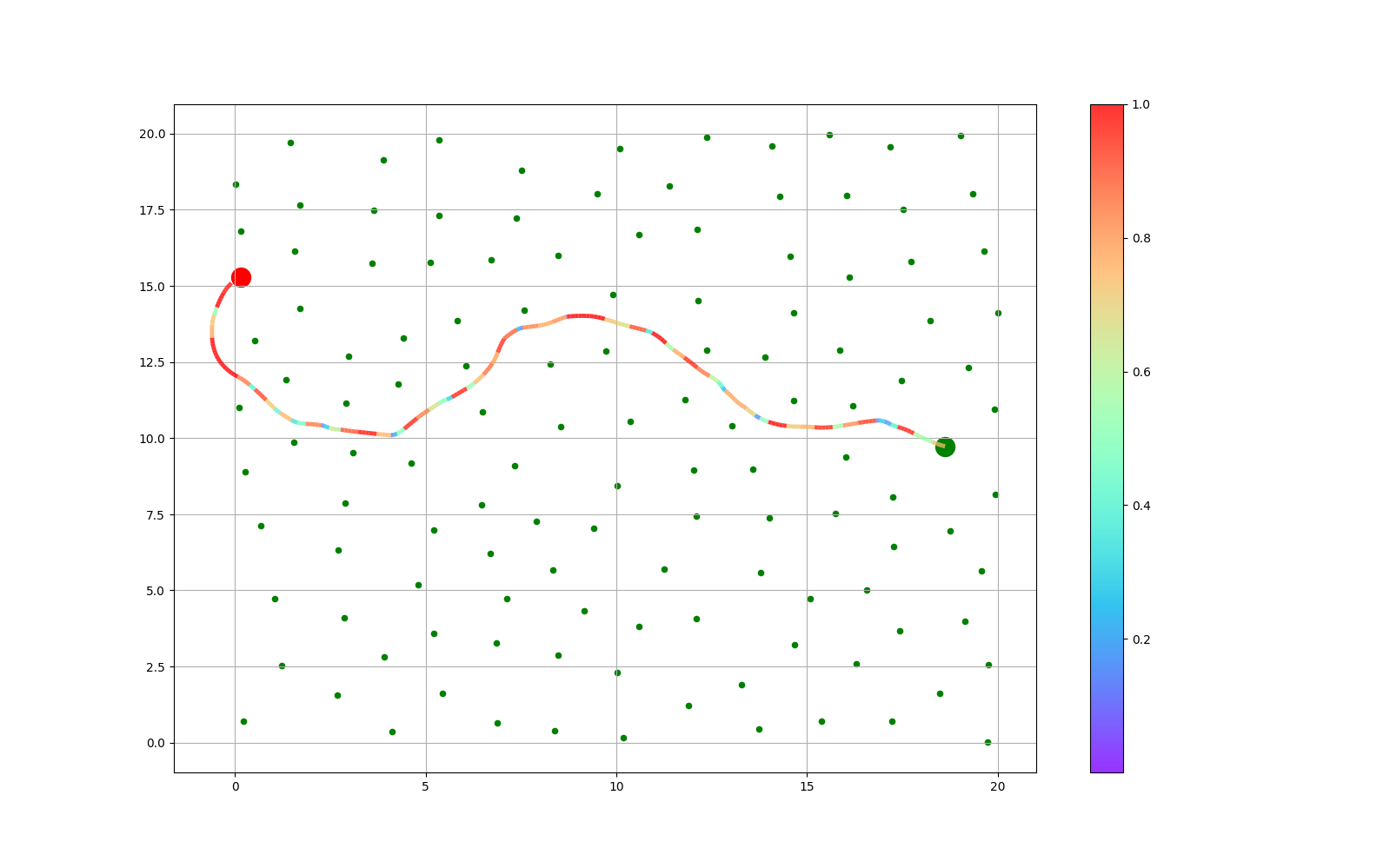}
\caption{Forest UGV Trial}
\label{fig:forest_trial}
\end{subfigure}
\caption{The BOW planner is evaluated for UGV navigation in Bugtrap environments and Poisson forest environments featuring minimum obstacle density of 1.5 per $m^2$.}
\label{fig:challenging_environments_uav}
\end{figure}
\vspace{-15pt}
\subsection{Benchmark}
{We evaluate the performance of the BOW planner against five established path planning algorithms: RRT~\cite{lavalle1998rapidly}, DWA~\cite{fox1997dynamic}, MPPI~\cite{williams2018information}, HRVO~\cite{snape2011hybrid}, and CBF~\cite{mann2024control}. For each environment and planner combination, we conducted 10 independent runs and calculated the mean and standard deviation of all key metrics except Time/Step to ensure statistical reliability. 

\begin{table*}[t]
    \centering
    \setlength{\tabcolsep}{5pt}
    \begin{tabular}{@{}llcrrrrrr@{}}
        \toprule
        \textbf{Env} & \textbf{Method} & \textbf{Lib + Lang} & \textbf{Steps} & \textbf{Traj. Length (m)} & \textbf{Total Time (ms)} & \textbf{Time / Step (ms)} & {\textbf{Avg Velocity}} & {\textbf{Avg Jerk}} \\
        \midrule
        \multirow{6}{*}{1}
            & BOW & FCL + C++ & 161.2 $\pm$ 2.9 & 14.14 $\pm$ 0.09 & \textbf{23.26 $\pm$ 0.35} & \textbf{0.14} & 0.947 $\pm$ 0.020 & 0.190 $\pm$ 0.024 \\
            & DWA & FCL + C++ & 157.0 $\pm$ 0.00 & 14.33 $\pm$ 0.00 & 103.68 $\pm$ 2.07 & 0.66 & 0.952 $\pm$ 0.00 & 0.031 $\pm$ 0.000 \\
            & HRVO & C++ & 539.0 $\pm$ 0.00 & 15.78 $\pm$ 0.00 & 269.22 $\pm$ 7.21 & 0.50 & 0.293 $\pm$ 0.00 & 0.223 $\pm$ 0.00 \\
            & RRT & FCL + C++ & \textbf{55.2 $\pm$ 8.0} & 17.39 $\pm$ 1.38 & 389.83 $\pm$ 272.41 & 7.06 & 0.796 $\pm$ 0.038 & 0.849 $\pm$ 0.062 \\
            & MPPI & CUDA + Py & 137.0 $\pm$ 0.00 & \textbf{13.98 $\pm$ 0.00} & 5865.90 $\pm$ 182.73 & 42.82 & \textbf{1.654 $\pm$ 0.052} & -1.150 $\pm$ 0.108 \\
            & CBF & Gurobi + Py & 156.0 $\pm$ 0.00 & 14.42 $\pm$ 0.00 & 9892.10 $\pm$ 97.26 & 63.41 & 0.992 $\pm$ 0.010 & \textbf{-0.020 $\pm$ 0.001} \\
        \midrule
        \multirow{5}{*}{2}
            & BOW & FCL + C++ & 160.5 $\pm$ 3.2 & \textbf{14.21 $\pm$ 0.08} & \textbf{23.09 $\pm$ 0.59} & \textbf{0.14} & 0.955 $\pm$ 0.022 & 0.188 $\pm$ 0.026 \\
            & DWA & FCL + C++ & 332.0 $\pm$ 0.00 & 30.26 $\pm$ 0.00 & 83.15 $\pm$ 1.78 & 0.25 & 1.032 $\pm$ 0.000 & \textbf{0.040 $\pm$ 0.000} \\
            & HRVO & C++ & 341.0 $\pm$ 0.00 & 14.87 $\pm$ 0.00 & 282.89 $\pm$ 7.65 & 0.83 & 0.436 $\pm$ 0.00 & 0.137 $\pm$ 0.00 \\
            & RRT & FCL + C++ & \textbf{57.0 $\pm$ 10.2} & 17.30 $\pm$ 2.56 & 464.82 $\pm$ 87.55 & 8.15 & 0.765$ \pm$ 0.036 & 0.859$ \pm$ 0.085\\
            & MPPI & CUDA + Py & 156.0 $\pm$ 0.00 & 14.38 $\pm$ 0.00 & 6625.00 $\pm$ 146.45 & 42.47 & \textbf{1.467 $\pm$ 0.032} & -2.647 $\pm$ 0.172 \\
            & CBF & Gurobi + Py & — & — & — & — & — & — \\
        \midrule
        \multirow{6}{*}{3}
            & BOW & FCL + C++ & 199.2 $\pm$ 3.8 & 17.48 $\pm$ 0.06 & \textbf{28.14 $\pm$ 0.58} & \textbf{0.14} & 0.939 $\pm$ 0.020 & 0.190 $\pm$ 0.016 \\
            & DWA & FCL + C++ & 220.0 $\pm$ 0.00 & 19.14 $\pm$ 0.00 & 102.50 $\pm$ 1.83 & 0.47 & 0.889 $\pm$ 0.000 & 0.028 $\pm$ 0.00 \\
            & HRVO & C++ & 422.0 $\pm$ 0.00 & 20.07 $\pm$ 0.00 & 468.44 $\pm$ 10.51 & 1.11 & 0.476 $\pm$ 0.00 & 0.302 $\pm$ 0.000 \\
            & RRT & FCL + C++ & \textbf{69.2 $\pm$ 12.7} & 20.69 $\pm$ 2.63 & 808.09 $\pm$ 533.78 & 11.68 & 0.783 $\pm$ 0.022 & 0.886 $\pm$ 0.050 \\
            & MPPI & CUDA + Py & 150.0 $\pm$ 0.00 & \textbf{17.27 $\pm$ 0.00} & 6496.80 $\pm$ 346.90 & 43.31 & \textbf{1.774 $\pm$ 0.089} & 2.062 $\pm$ 0.291 \\
            & CBF & Gurobi + Py & 151.0 $\pm$ 0.00 & 17.44 $\pm$ 0.00 & 12146.60 $\pm$ 185.49 & 80.44 & 0.930 $\pm$ 0.014 & \textbf{-0.018 $\pm$ 0.001} \\
        \midrule
        \multirow{6}{*}{4}
            & BOW & FCL + C++ & 210.8 $\pm$ 6.5 & 18.67 $\pm$ 0.12 & \textbf{30.90 $\pm$ 0.68} & \textbf{0.15} & 0.945 $\pm$ 0.036 & 0.172 $\pm$ 0.028 \\
            & DWA & FCL + C++ & 373.0 $\pm$ 0.00 & 34.33 $\pm$ 0.00 & 169.20 $\pm$ 3.58 & 0.45 & 0.970 $\pm$ 0.00 & 0.030 $\pm$ 0.00 \\
            & HRVO & C++ & 465.0 $\pm$ 0.00 & 21.06 $\pm$ 0.00 & 676.32 $\pm$ 12.82 & 1.45 & 0.453 $\pm$ 0.00 & 0.341 $\pm$ 0.000 \\
            & RRT & FCL + C++ & \textbf{77.9 $\pm$ 10.0} & 23.68 $\pm$ 2.67 & 745.79 $\pm$ 406.64 & 9.57 & 0.769 $\pm$ 0.021 & 0.898 $\pm$ 0.072 \\
            & MPPI & CUDA + Py & 198.0 $\pm$ 0.00 & 18.86 $\pm$ 0.00 & 8674.80 $\pm$ 156.41 & 43.81 & \textbf{1.467 $\pm$ 0.026} & 0.194 $\pm$ 0.010 \\
            & CBF & Gurobi + Py & 171.0 $\pm$ 0.00 & \textbf{18.09 $\pm$ 0.00} & 15135.40 $\pm$ 335.45 & 88.51 & 0.805 $\pm$ 0.018 & \textbf{-0.002 $\pm$ 0.000} \\
        \midrule
        \multirow{4}{*}{5}
            & BOW & FCL + C++ & 95.0 $\pm$ 5.4 & \textbf{7.56 $\pm$ 0.19} & \textbf{12.89 $\pm$ 0.09} & \textbf{0.14} & \textbf{1.007 $\pm$ 0.077} & 0.357 $\pm$ 0.052 \\
            & DWA & FCL + C++ & — & — & — & — & — & — \\
            & HRVO & C++ & 204.0 $\pm$ 0.00 & 8.34 $\pm$ 0.00 & 369.20 $\pm$ 2.49 & 1.81 & 0.409 $\pm$ 0.000 & \textbf{0.260 $\pm$ 0.00} \\
            & RRT & FCL + C++ & \textbf{31.7 $\pm$ 3.7} & 9.53 $\pm$ 0.80 & 93.18 $\pm$ 56.97 & 2.94 & 0.785 $\pm$ 0.035 & 0.927 $\pm$ 0.100 \\
            & MPPI & CUDA + Py & 139.0 $\pm$ 0.00 & 8.86 $\pm$ 0.00 & 5983.00 $\pm$ 183.18 & 43.04 & 0.691 $\pm$ 0.021 & -0.930 $\pm$ 0.086 \\
            & CBF & Gurobi + Py & — & — & — & — & — & — \\
        \midrule
        \multirow{4}{*}{6}
            & BOW & FCL + C++ & 101.9 $\pm$ 5.1 & \textbf{8.73 $\pm$ 0.09} & \textbf{10.00 $\pm$ 0.33} & \textbf{0.10} & \textbf{0.940 $\pm$ 0.063} & 0.158 $\pm$ 0.029 \\
            & DWA & FCL + C++ & 556.0 $\pm$ 0.00 & 25.36 $\pm$ 0.00 & 589.78 $\pm$ 4.04 & 1.06 & 0.556 $\pm$ 0.00 & \textbf{0.043 $\pm$ 0.00} \\
            & HRVO & C++ & — & — & — & — & — & — \\
            & RRT & FCL + C++ & \textbf{36.1 $\pm$ 8.1} & 10.54 $\pm$ 1.56 & 128.16 $\pm$ 79.40 & 3.55 & 0.757 $\pm$ 0.041 & 0.811 $\pm$ 0.060 \\
            & MPPI & CUDA + Py & 1048.0 $\pm$ 0.00 & 29.94 $\pm$ 0.00 & 44312.20 $\pm$ 645.95 & 42.28 & 0.056 $\pm$ 0.001 & 0.421 $\pm$ 0.018 \\
            & CBF & Gurobi + Py & — & — & — & — & — & — \\
        \bottomrule
    \end{tabular}
    \caption{Comparison of path planning methods across six environments. Best values for each environment are highlighted in \textbf{bold}. A dash (—) indicates that no solution exists for that planner in the specific environment.}
    \label{tab:comparison}
    \vspace{-10pt}
\end{table*}
Evaluation across six distinct environments, as shown in Fig. \ref{fig:environments}, reveals that the BOW planner, using only 5 samples per planning horizon, consistently achieves the fastest total planning times \textemdash ranging from 10.00 $\pm$ 0.33 ms to 30.90 $\pm$ 0.68 ms, and per-step execution times between 0.10 and 0.15 ms, while maintaining competitive safe-trajectory lengths from 7.56 $\pm$ 0.19 m to 18.67 $\pm$ 0.12 m.
In contrast, while MPPI produced the shortest paths in Envs. 1 and 3 with trajectories of 13.98 $\pm$ 0.00 m and 17.27 $\pm$ 0.00 m  respectively, and CBF excelled in Env. 4 with 18.09 $\pm$ 0.00 m, both methods showed substantially higher computational demands with CBF requiring 9892.10 $\pm$ 97.26 ms to 15135.40 $\pm$ 335.45 ms total planning time and 63.41 to 88.51 ms per step, while MPPI demanded 5865.90 $\pm$ 182.73 ms to 44312.20 $\pm$ 645.95 ms total time and 42.28 to 43.81 ms per step, highlighting significant efficiency trade-offs. Critical completeness issues emerged across algorithms, with CBF failing to generate solutions in Envs. 2, 5, and 6, DWA unable to solve Env. 5, and HRVO failing in Env. 6, while BOW maintained perfect completeness across all test scenarios. RRT consistently required the fewest planning steps ranging from 31.7 $\pm$ 3.7 to 77.9 $\pm$ 10.0 but produced longer trajectories, while MPPI showed poor scalability, requiring up to 1048.0 $\pm$ 0.00 steps in Env. 6 compared to 137.0 $\pm$ 0.00 in simpler environments. Additionally, velocity and smoothness characteristics varied significantly across planners, with MPPI achieving the highest average velocities in Envs. 1-4 ranging from 1.467 $\pm$ 0.026 to 1.774 $\pm$ 0.089 m/s, while BOW maintained consistent velocity performance between 0.939 $\pm$ 0.020 to 1.007 $\pm$ 0.077 m/s across all environments. 
In terms of trajectory smoothness, CBF achieved the lowest jerk values ($-0.020 \pm 0.001$ to $-0.002 \pm 0.000$), indicating very smooth paths, followed by DWA ($0.028 \pm 0.000$ to $0.043 \pm 0.000$). BOW showed moderate smoothness ($0.158 \pm 0.029$ to $0.357 \pm 0.052$), while MPPI exhibited highly variable performance ($-2.647 \pm 0.172$ to $2.062 \pm 0.291$). RRT consistently had the highest jerk values ($0.811 \pm 0.060$ to $0.927 \pm 0.100$), indicating the least smooth trajectories.
These results show that the BOW planner is well-suited for resource-constrained, real-time applications requiring robust performance in diverse environments.
}

\begin{figure*}[t]
    \centering
    \includegraphics[width=0.19\linewidth,trim={0 20pt 0 0pt},clip]{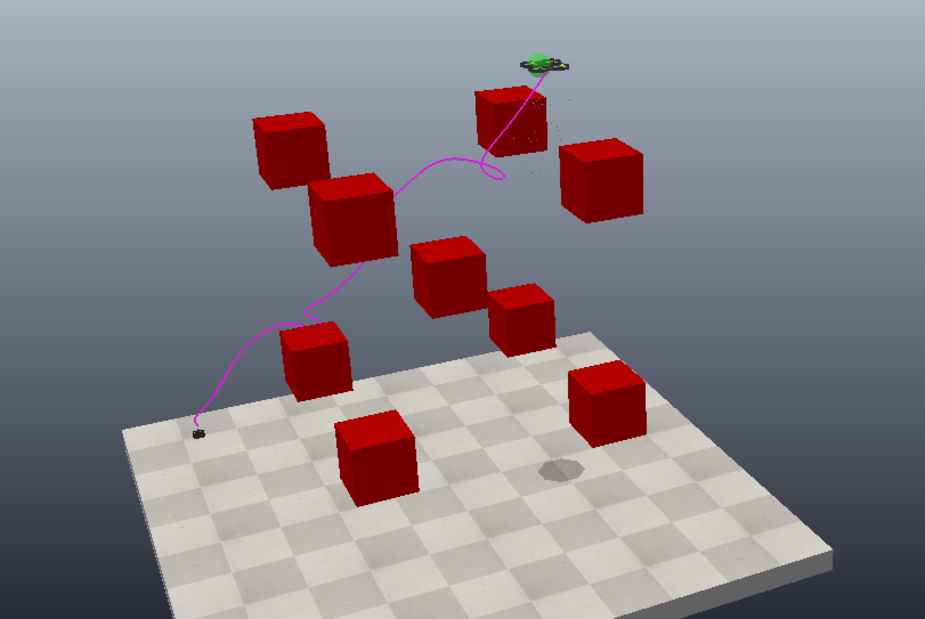} \hfill
    \includegraphics[width=0.19\linewidth,trim={0 30pt 0pt 10pt},clip]{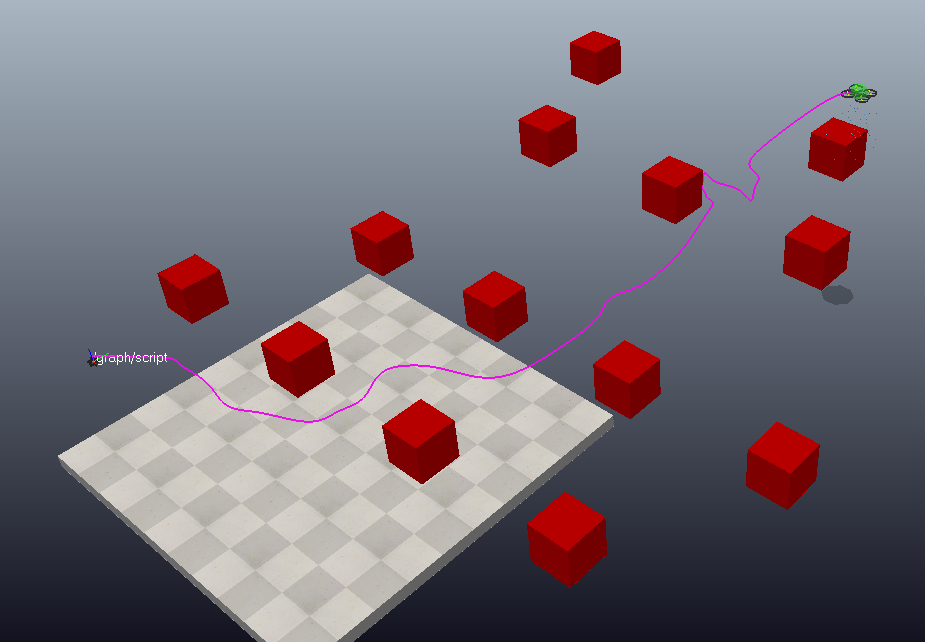} \hfill
    \includegraphics[width=0.19\linewidth,trim={0pt 0pt 0 10pt},clip]{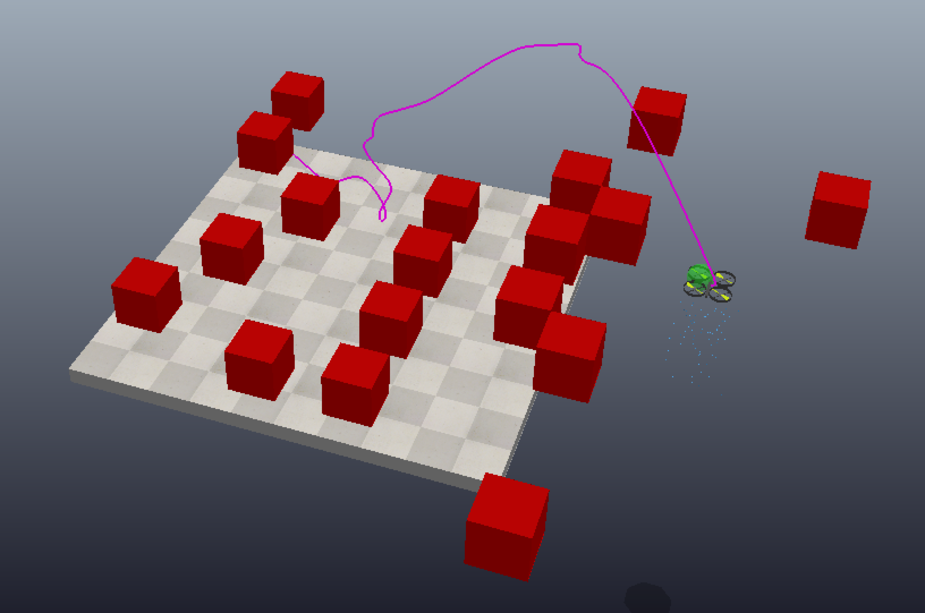} \hfill
    \includegraphics[width=0.19\linewidth,trim={0pt 0pt 0 0pt},clip]{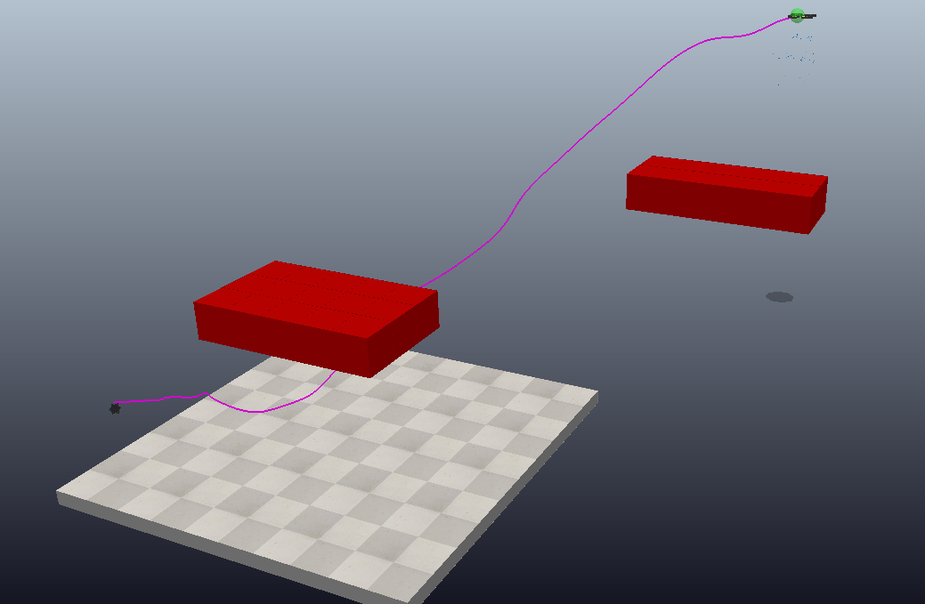} \hfill
    \includegraphics[width=0.19\linewidth,trim={50pt 130pt 0 80pt},clip]{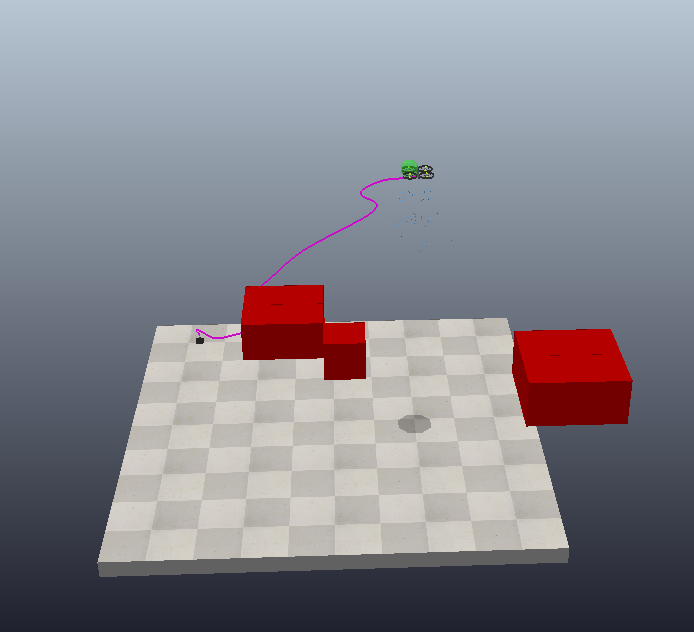} 
    \caption{Simulated 3D trajectory planning scenarios for a UAV using the BOW planner in Coppeliasim Simulator, showing five different environments.}
    \label{fig:bow_drn}
    \vspace{-15 pt}
\end{figure*}
\vspace{-5pt}
\subsection{Unmanned Aerial Vehicle Experiments}
To showcase the BOW planner's scalability in high-dimensional motion planning problems, we applied it to a 6-DOF aerial robot operating in a 3-D environment with 3-D obstacles (Fig.~\ref{fig:bow_drn}). The robot's state is represented by a 12-D vector $\mathbf{x} = (x, y, z, \theta, \dot{x}, \dot{y}, \dot{z}, \dot{\theta}, \ddot{x}, \ddot{y}, \ddot{z}, \ddot{\theta})$, encompassing position, orientation, velocity, and acceleration. The control input is a 4-D vector $\mathbf{u} = (\dot{x}_c, \dot{y}_c, \dot{z}_c, \dot{\theta}_c)$, consisting of linear velocities and yaw rate relative to the body frame. We evaluated the BOW planner using a quadrotor UAV in both simulations and real-world experiments.
Fig.~\ref{fig:bow_drn} depicts five distinct simulated environments, with red cubes representing obstacles. In each scenario, an Unmanned Aerial Vehicle (UAV) starts from an initial position and navigates to a goal location using the BOW planner. These simulations were conducted in the CoppeliaSim Simulator.

Fig.~\ref{fig:real_obs} illustrates the implementation of the BOW planner in a UAV's navigation through two {different problems in terms of start and goal locations. Although the BOW planner is lightweight enough to run on edge devices like the Raspberry Pi or Jetson Nano, it was executed on a desktop due to the off-the-shelf UAV used, with control commands sent to the UAV during the experiments.} In these scenarios, obstacles are represented by yellow cubes. The goal locations are marked with red labels ``\textbf{G1}" and ``\textbf{G2}", while the start locations are denoted by rose-colored labels ``\textbf{S1}" and ``\textbf{S2}". The figure shows the UAV's path as it successfully avoids obstacles to reach its destination.

\begin{figure}[tb]
    \centering
    \begin{subfigure}[b]{0.48\linewidth} 
      \includegraphics[width=\linewidth, height=3.5cm]{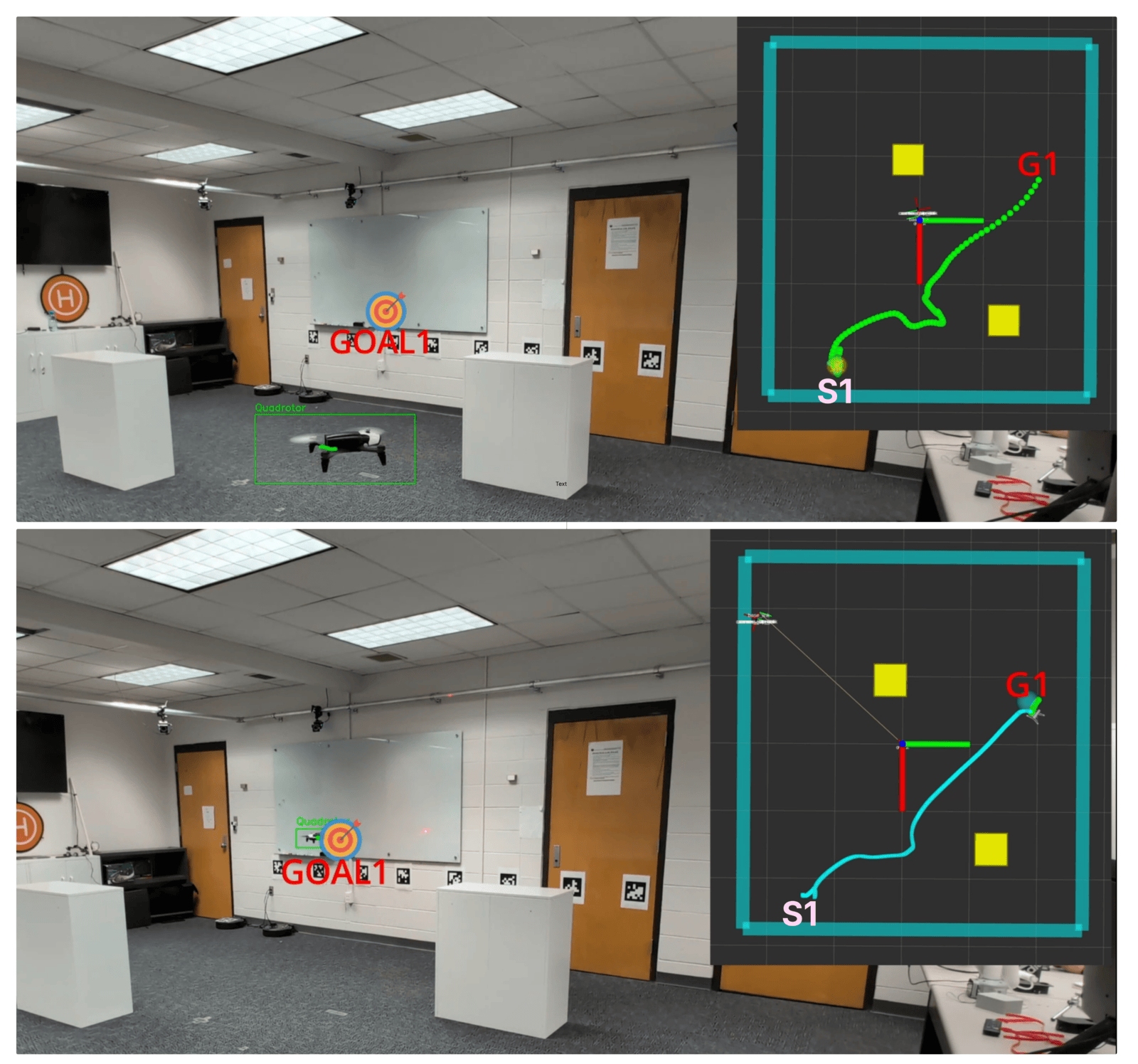}
      \centering
      \caption{}
      \label{fig:drn_real1}
    \end{subfigure}
    ~
    \begin{subfigure}[b]{0.48\linewidth} 
      \includegraphics[width=\linewidth, height=3.5cm]{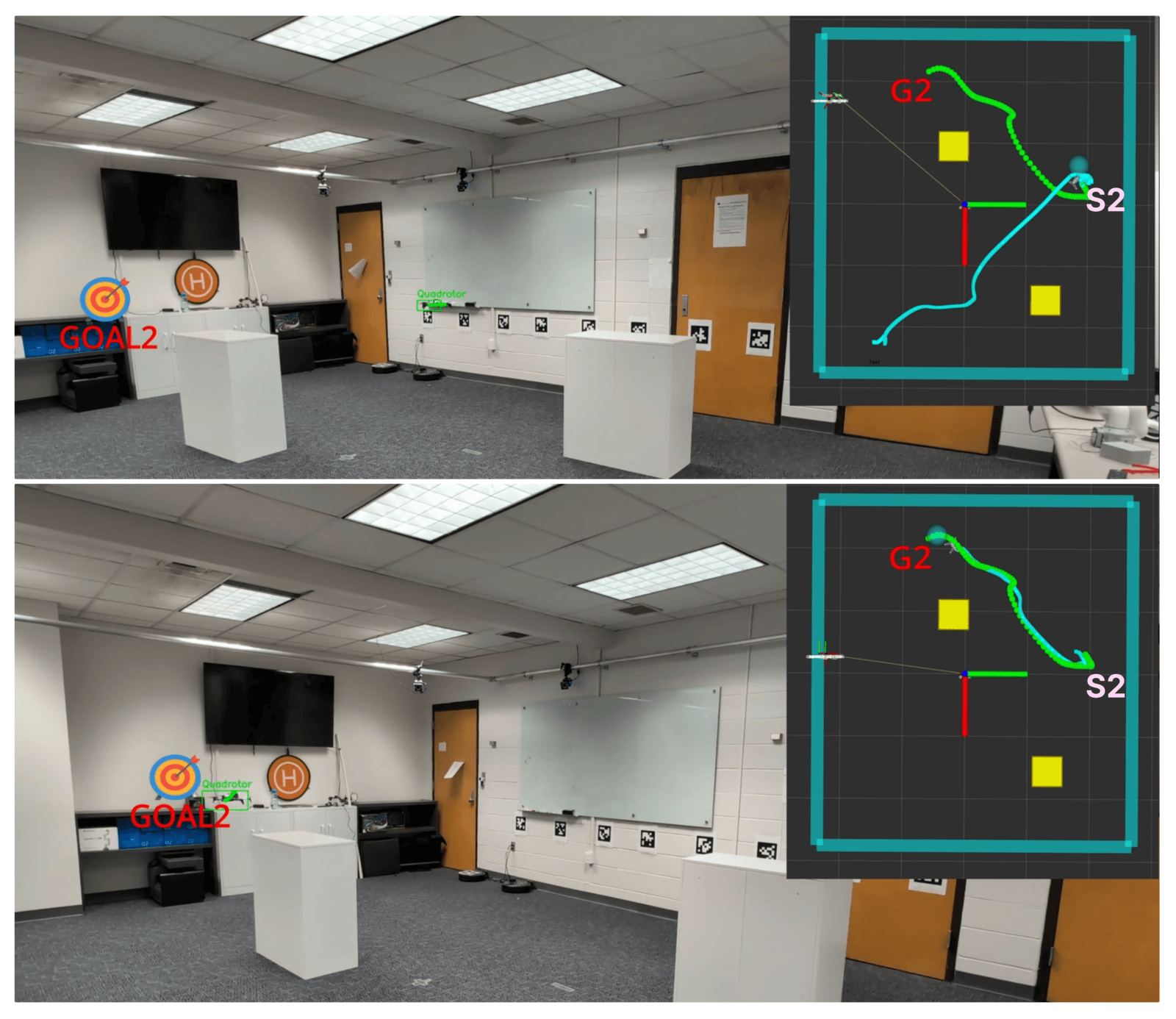}
      \centering
      \caption{}
      \label{fig:drn_real2}
    \end{subfigure}
\caption{
    The BOW planner can enhance the autonomous navigation of UAV in real-world scenarios.
 }
\label{fig:real_obs}
\end{figure}
Fig.~\ref{fig:drn_real1} and Fig.~\ref{fig:drn_real2} illustrate the UAV's generated path from start to goal location in two real environments with obstacles.
These results demonstrate the UAV’s obstacle avoidance capabilities and its ability to navigate to specific destinations in a 3D environment.

\begin{table}[hbt]
    \centering
    \renewcommand{\arraystretch}{0.9} 
    \setlength{\tabcolsep}{1pt} 
    \begin{tabular}{lccccc}
        \toprule
        \textbf{Env.} & \textbf{Density} & \textbf{Dispersity} & \textbf{Total Time (ms)} & \textbf{Path Length (m)} & \textbf{Steps} \\
        \midrule
        1 & 0.14  & 2.38 & 10.81 $\pm$ 0.43 & 6.73 $\pm$ 1.14  & 132.70 $\pm$ 33.95 \\
        2 & 0.04  & 4.15 & 12.07 $\pm$ 0.31 & 10.76 $\pm$ 0.90 & 147.90 $\pm$ 25.30 \\
        3 & 0.05  & 3.35 & 12.84 $\pm$ 0.49 & 10.57 $\pm$ 1.01 & 152.60 $\pm$ 30.13 \\
        4 & 0.12  & 3.35 & 11.71 $\pm$ 0.47 & 10.77 $\pm$ 0.57 & 151.90 $\pm$ 23.53 \\
        5 & 0.09  & 2.63 & 5.64 $\pm$ 0.18  & 5.78 $\pm$ 1.60  & 89.50 $\pm$ 38.44  \\
        \bottomrule
    \end{tabular}
    \caption{Performance Analysis of the BOW Planner for Quadrotor UAV Navigation Across Varied Obstacle Densities and Dispersities.}\label{tab:quadrotor}
    \label{tab:results}
\end{table}
Table~\ref{tab:quadrotor} summarizes the BOW planner's performance for a quadrotor UAV across five environments (Fig.~\ref{fig:bow_drn}) with varying obstacle densities and dispersities. Higher occupied area ratios (e.g., Env. 1: 0.14; Env. 4: 0.12) led to shorter paths but moderate computation and node expansion. In contrast, lower occupied area ratios (e.g., Env. 2: 0.04) resulted in longer paths and greater exploration. High dispersity (e.g., Env. 2: 4.15) increased planning time, indicating greater global search. Env. 5 showed the lowest planning time (5.64~ms), likely due to moderate density (0.09) and dispersity (2.63). Overall, BOW performs best in environments with moderate density and dispersity, balancing path quality and efficiency.

To further validate our methodology, we conducted additional evaluations in complex environments, specifically Bugtrap and Poisson forest settings, as illustrated in Fig.~\ref{fig:challenging_environments_uav}. Complete experimental results for both terrestrial and aerial platforms can be found on our project website.
     
    \section{Conclusion}

We introduced the BOW planner, a novel contribution to local motion planning in complex, non-convex, and high-dimensional environments through the integration of constrained Bayesian optimization. Grounded in a proven theoretical foundation, the method combines a window-based motion planner with constrained Bayesian optimization to produce near-optimal control policies while minimizing evaluations of computationally expensive constrained objectives. Extensive simulation benchmarks and real-world experiments demonstrate that the BOW planner achieves notable improvements in sample efficiency, trajectory length, and computation time compared to a range of existing local planners. Its deployment across diverse robotic platforms further highlights its online adaptability and safety-aware optimization. The planner is released as open-source software, facilitating adoption and extension by researchers and practitioners engaged in high-dimensional motion planning.

Like most local planners, the BOW planner employs a cost-to-go heuristic to efficiently find feasible trajectories, which may result in getting stuck in local minima when the heuristic is inadmissible (e.g., in narrow passage problems). Future research will focus on combining tree-based search techniques with the BOW planner to address global motion planning problems.

    \bibliography{main}
    \bibliographystyle{IEEEtran}

\end{document}